\documentclass[smallextended]{svjour3}       
\smartqed  
\usepackage{graphicx}

\usepackage{utopia}
\usepackage{amssymb}
\usepackage{amsfonts}
\usepackage{amsmath}\allowdisplaybreaks
\usepackage{mathtools}
\usepackage{algorithm}
\usepackage{algorithmic}
\usepackage{multirow}
\usepackage{enumerate}
\usepackage{bm}
\usepackage{algorithmic}
\usepackage{hyperref}
\usepackage[american]{babel}
\mathtoolsset{showonlyrefs,showmanualtags}

\newenvironment{myproof}
{\par\noindent\textit{Proof}\ \enspace\ignorespaces\begin{allowdisplaybreaks}}
{\end{allowdisplaybreaks}\hspace{\stretch{1}}$\square$}

\begin{document}

\title{Analysis of Noisy Evolutionary Optimization When Sampling Fails\thanks{A preliminary version of this paper has appeared at GECCO'18~\cite{qian2018gecco}.}
}
\subtitle{}

\titlerunning{Analysis of Noisy Evolutionary Optimization When Sampling Fails}        

\author{Chao Qian$^{1}$         \and
        Chao Bian$^{1}$         \and
        Yang Yu$^{1}$         \and
        Ke Tang$^{2}$     \and
        Xin Yao$^{2}$
}

\authorrunning{Chao Qian et al.} 


\institute{1\quad National Key Laboratory for Novel Software Technology, Nanjing University, Nanjing 210023, China\vspace{0.3em}\\
2\quad Department of Computer Science and Engineering, Southern University of Science and Technology, Shenzhen 518055, China}

\date{Received: date / Accepted: date}

\maketitle

\begin{abstract}
In noisy evolutionary optimization, sampling is a common strategy to deal with noise. By the sampling strategy, the fitness of a solution is evaluated multiple times (called \emph{sample size}) independently, and its true fitness is then approximated by the average of these evaluations. Most previous studies on sampling are empirical, and the few theoretical studies mainly showed the effectiveness of sampling with a sufficiently large sample size. In this paper, we theoretically examine what strategies can work when sampling with any fixed sample size fails. By constructing a family of artificial noisy examples, we prove that sampling is always ineffective, while using parent or offspring populations can be helpful on some examples. We also construct an artificial noisy example to show that when using neither sampling nor populations is effective, a tailored adaptive sampling (i.e., sampling with an adaptive sample size) strategy can work. These findings may enhance our understanding of sampling to some extent, but future work is required to validate them in natural situations.
\end{abstract}

\section{Introduction}

Evolutionary algorithms (EAs) are a type of general-purpose randomized optimization algorithms, inspired by natural evolution. They have been applied to solve various real-world optimization problems~\cite{li2018path,mukhopadhyay2013survey,xu2019efficient,zhou2019evolutionary}, which are often subject to noise. Sampling is a popular strategy for dealing with noise: to estimate the fitness of a solution, it evaluates the fitness multiple ($m$) times (called \emph{sample size}) independently and then uses the sample average to approximate the true fitness. Sampling reduces the variance of noise by a factor of $m$, but also increases the computation time for the fitness estimation of a solution by $m$ times. Previous studies mainly focused on the empirical design of efficient sampling methods, e.g., adaptive sampling~\cite{branke2004sequential,cantu2004adaptive}, which dynamically decides the sample size $m$ for each solution in each generation. The theoretical analysis on sampling was rarely touched.

Due to their sophisticated behaviors of mimicking natural phenomena, the theoretical analysis of EAs is difficult. Much effort thus has been devoted to understanding the behavior of EAs from a theoretical viewpoint~\cite{auger2011theory,neumann2010bioinspired}, but most of such works focus on noise-free optimization. The presence of noise further increases the randomness of optimization, and thus also increases the difficulty of analysis.

For running time analysis (one essential theoretical aspect) in noisy evolutionary optimization, only a few results have been reported. The classic (1+1)-EA algorithm was first studied on the OneMax and LeadingOnes problems under various noise models~\cite{qian2018ppsn,doerr2018gecco,droste2004analysis,giessen2014robustness,qian2018noise,dirk2018gecco}, including one-bit noise which flips a random bit of a binary solution before evaluation with probability $p \in [0,1]$, and additive Gaussian noise which adds a value randomly drawn from the Gaussian distribution. The results showed that the (1+1)-EA is efficient only under low noise levels, e.g., for the (1+1)-EA solving OneMax in the presence of one-bit noise, the maximal noise level of allowing a polynomial running time is $O((\log n)/n)$, where the noise level is characterized by the noise probability $p$, and $n$ is the problem size. Later studies mainly proved the robustness of different strategies to noise, including using populations~\cite{dang2015efficient,doerr2018gecco,giessen2014robustness,prugel2015run,dirk2018gecco}, sampling~\cite{qian2018noise,qian2016sampling} and threshold selection~\cite{qian2015noise}. For example, the ($\mu$+1)-EA with $\mu \geq 12\ln(15n)$~\cite{giessen2014robustness}, the (1+$\lambda$)-EA with $\lambda \geq 24n\ln n$~\cite{giessen2014robustness}, the (1+1)-EA using sampling with $m=3$~\cite{qian2016sampling} or the (1+1)-EA using threshold selection with threshold $\tau=1$~\cite{qian2015noise} can solve OneMax in polynomial time even if the probability of one-bit noise reaches $1$. Note that there was also a sequence of papers analyzing the running time of the compact genetic algorithm~\cite{friedrich2015benefit} and ant colony optimization algorithms~\cite{doerr2012ants,feldmann2013optimizing,friedrich2015robustness,sudholt2012simple} solving noisy problems, including OneMax as well as a combinatorial optimization problem, single destination shortest paths. Recently, Qian et al.~\cite{qian2019distributed,qian2017subset} proved the polynomial-time approximation guarantee of simple multi-objective EAs for solving a general problem, subset selection, under additive or multiplicative noise, and showed that the algorithms can be easily distributed for large-scale applications.

The very few running time analyses involving sampling~\cite{qian2018noise,qian2016sampling} mainly showed the effectiveness of sampling with a large enough fixed sample size~$m$. For example, for the (1+1)-EA solving OneMax under one-bit noise with $p=\omega((\log n)/n)$, using sampling with $m=4n^3$ can reduce the running time from super-polynomial to polynomial. In addition, Akimoto et al.~\cite{akimoto2015analysis} proved that using sampling with a large enough $m$ can make optimization under additive unbiased noise behave as noiseless optimization. However, there are still many fundamental theoretical issues that have not been addressed, e.g., what strategies can work when sampling fails.

In this paper, we theoretically compare the two strategies of using populations and sampling on the robustness to noise. Previous studies have shown that both of them are effective for solving OneMax under one-bit noise~\cite{giessen2014robustness,qian2018noise,qian2016sampling}, while using sampling is better for solving OneMax under additive Gaussian noise~\cite{qian2016sampling}. Here, we complement this comparison by constructing a family of artificial noisy OneMax problems, and showing that using parent or offspring populations can be better than using sampling on some problems in this family. We also prove that the employed parent and offspring population sizes are almost tight.

Furthermore, we give an artificial noisy OneMax problem where using neither populations nor sampling is effective. For this case, we prove that using adaptive sampling can reduce the running time from exponential to polynomial, providing some theoretical justification for the good empirical performance of adaptive sampling~\cite{syberfeldt2010evolutionary,zhang2007immune}.

This paper extends our preliminary work~\cite{qian2018gecco}. When comparing sampling with populations, we only considered parent populations in~\cite{qian2018gecco}. To get a complete understanding, we add the analysis of using offspring populations (i.e., Section~\ref{sec-offspring}), showing that using offspring populations can be better than using sampling (i.e., Theorem~\ref{theo-population-offpsring} in Section~\ref{sec-offspring}). For the artificial noisy example in Section~\ref{sec-adaptive}, where we previously proved that using neither sampling nor parent populations is effective while adaptive sampling can work, we now prove that using offspring populations is also ineffective (i.e., Theorem~\ref{theo-offspring-adaptive} in Section~\ref{sec-adaptive}). To show that using parent populations is better than using sampling, we only gave an effective parent population size in~\cite{qian2018gecco}. We now add the analysis of the tightness of the effective parent population size (i.e., Theorem~\ref{theo-parent-population-2} in Section~\ref{sec-parent}) as well as the effective offspring population size (i.e., Theorem~\ref{theo-offspring-population-2} in Section~\ref{sec-offspring}).

In~\cite{qian2018gecco}, we also analyzed the (1+1)-EA solving the LeadingOnes problem under one-bit noise with $p=1$, which always flips a random bit of a binary solution before evaluation. We showed that as the sample size~$m$ increases, the expected running time to find the optimal solution, i.e., the string with all 1s, can reduce from exponential to polynomial, but then return to exponential. Note that we delete this part here due to the ill-defined noisy setting. Under one-bit noise with $p=1$, the expected fitness of the string with all 1s is no longer the largest.

The rest of this paper is organized as follows. Section~\ref{sec-prelimaries} introduces some preliminaries. The effectiveness of using populations when sampling fails is proved in Section~\ref{sec-population}. Section~\ref{sec-adaptive} then shows that when using neither sampling nor populations is effective, adaptive sampling can work. Finally, Section~\ref{sec-conclusion} concludes the paper.

\section{Preliminaries}\label{sec-prelimaries}

In this section, we first introduce the EAs and the sampling strategy, and then present the analysis tools that will be used in this paper.

\subsection{Evolutionary Algorithms}

\begin{algorithm}[t]\caption{(1+1)-EA}\label{(1+1)-EA} Given a pseudo-Boolean function $f: \{0,1\}^n \rightarrow \mathbb{R}$ to be maximized, the procedure of the (1+1)-EA is:
    \begin{algorithmic}[1]
    \STATE  Let $x$ be a uniformly chosen solution from $\{0,1\}^n$.
    \STATE  Repeat until some termination condition is met
    \STATE  \quad $x':=$ copy $x$ and flip each bit independently with probability $1/n$.
    \STATE \quad if {$f(x') \geq f(x)$} then $x:=x'$.
    \end{algorithmic}
\end{algorithm}

\begin{algorithm}[t]\caption{($\mu$+1)-EA}\label{(mu+1)-EA} Given a pseudo-Boolean function $f: \{0,1\}^n \rightarrow \mathbb{R}$ to be maximized, the procedure of the ($\mu$+1)-EA is:
    \begin{algorithmic}[1]
    \STATE Let $P$ be a set of $\mu$ uniformly chosen solutions from $\{0,1\}^n$.
    \STATE Repeat until some termination condition is met
    \STATE\quad $x:=$ uniformly selected from $P$ at random.
    \STATE\quad $x':=$ copy $x$ and flip each bit independently with probability $1/n$.
    \STATE\quad Let $z \in \arg\min_{z \in P} f(z)$; ties are broken randomly.
    \STATE\quad if {$f(x') \geq f(z)$} then $P:=(P\setminus\{z\})\cup \{x'\}$.
    \end{algorithmic}
\end{algorithm}

\begin{algorithm}[t]\caption{(1+$\lambda$)-EA}\label{(1+lambda)-EA} Given a pseudo-Boolean function $f: \{0,1\}^n \rightarrow \mathbb{R}$ to be maximized, the procedure of the (1+$\lambda$)-EA is:
    \begin{algorithmic}[1]
    \STATE Let $x$ be a uniformly chosen solution from $\{0,1\}^n$.
    \STATE Repeat until some termination condition is met
    \STATE \quad Let $ Q:=\emptyset $.
    \STATE \quad for $ i=1 $ to $ \lambda $ do
    \STATE \qquad $x':=$ copy $x$ and flip each bit independently with probability $1/n$.
    \STATE \qquad $ Q:=Q\cup \{x'\} $.
    \STATE \quad Let $z \in \arg\max_{z \in Q} f(z)$; ties are broken randomly.
    \STATE \quad if {$f(z) \geq f(x)$} then $x:=z$.
    \end{algorithmic}
\end{algorithm}

The (1+1)-EA (i.e., Algorithm~\ref{(1+1)-EA}) maintains only one solution, and iteratively tries to produce one better solution by bit-wise mutation and selection. The ($\mu$+1)-EA (i.e., Algorithm~\ref{(mu+1)-EA}) uses a parent population size $\mu$. In each iteration, it also generates one new solution $x'$, and then uses $x'$ to replace the worst solution in the population $P$ if $x'$ is not worse. The (1+$\lambda$)-EA (i.e., Algorithm~\ref{(1+lambda)-EA}) uses an offspring population size $\lambda$. In each iteration, it generates $\lambda$ offspring solutions independently by mutating the parent solution $x$, and then uses the best offspring solution to replace the parent solution if it is not worse. When $\mu=1$ and $\lambda=1$, both the ($\mu$+1)-EA and (1+$\lambda$)-EA degenerate to the (1+1)-EA. Note that for the ($\mu$+1)-EA, a slightly different updating rule is also used~\cite{friedrich2015benefit,witt2006runtime}: $x'$ is simply added into $P$ and then the worst solution in $P \cup \{x'\}$ is deleted. Our results about the ($\mu$+1)-EA derived in the paper also apply to this setting.

In noisy optimization, only a noisy fitness value $f^{\mathrm{n}}(x)$ instead of the exact one $f(x)$ can be accessed. Note that in our analysis, the algorithms are assumed to use the reevaluation strategy as in~\cite{doerr2012ants,droste2004analysis,giessen2014robustness}. That is, besides evaluating the noisy fitness $f^{\mathrm{n}}(x')$ of offspring solutions, the noisy fitness values of parent solutions will be reevaluated in each iteration. The running time of EAs is usually measured by the number of fitness evaluations until finding an optimal solution w.r.t. the true fitness function $f$ for the first time~\cite{akimoto2015analysis,droste2004analysis,giessen2014robustness}.

\subsection{Sampling}

Sampling as described in Definition~\ref{sampling} is a common strategy to deal with noise. It approximates the true fitness $f(x)$ using the average of a number of random evaluations. The number $m$ of random evaluations is called the \emph{sample size}. Note that $m=1$ implies that sampling is not used. Qian et al.~\cite{qian2018noise,qian2016sampling} have theoretically shown the robustness of sampling to noise. Particularly, they proved that by using sampling with some fixed sample size, the running time of the (1+1)-EA for solving OneMax and LeadingOnes under noise can reduce from exponential to polynomial.

\begin{definition}[Sampling]\label{sampling}
Sampling first evaluates the fitness of a solution $m$ times independently and obtains the noisy fitness values $f^{\mathrm{n}}_1(x),f^{\mathrm{n}}_2(x),\ldots,$ $f^{\mathrm{n}}_m(x)$, and then outputs their average, i.e., $
 \hat{f}(x)=\sum\nolimits^{m}_{i=1} f^{\mathrm{n}}_i(x)/m.
$
\end{definition}

Adaptive sampling dynamically decides the sample size for each solution in the optimization process, instead of using a fixed size. For example, one popular strategy~\cite{branke2004sequential,cantu2004adaptive} is to first estimate the fitness of two solutions by a small number of samples, and then sequentially increase samples until the difference can be significantly discriminated. It has been found well useful in many applications~\cite{syberfeldt2010evolutionary,zhang2007immune}, while there has been no theoretical work supporting its effectiveness.

\subsection{Analysis Tools}

EAs often generate offspring solutions only based on the current population, thus, an EA can be modeled as a Markov chain $\{\xi_t\}^{+\infty}_{t=0}$ (e.g., in~\cite{he2001drift,yu2014switch}) by taking the EA's population space $\mathcal{X}$ as the chain's state space (i.e., $\xi_t \in \mathcal{X}$) and taking the set $\mathcal{X}^*$ of all optimal populations as the chain's target state space. Note that the population space $\mathcal{X}$ consists of all possible populations, and an optimal population contains at least one optimal solution.

Given a Markov chain $\{\xi_t\}^{+\infty}_{t=0}$ and the state $\xi_{\hat{t}}$ at time $\hat{t}$, we define its \emph{first hitting time} starting from $\xi_{\hat{t}}$ as $\tau=\min\{t \mid \xi_{\hat{t}+t} \in \mathcal{X}^*,t\geq0\}$. The expectation of $\tau$, $\mathrm{E}(\tau \mid \xi_{\hat{t}})=\sum\nolimits^{+\infty}_{i=0} i\cdot\mathrm{P}(\tau=i \mid \xi_{\hat{t}})$, is called the \emph{expected first hitting time} (EFHT). If $\xi_{0}$ is drawn from a distribution $\pi_{0}$, $\mathrm{E}(\tau \mid \xi_{0}\sim \pi_0) = \sum\nolimits_{\xi_0\in \mathcal{X}} \pi_{0}(\xi_0)\cdot \mathrm{E}(\tau \mid \xi_{0})$ is called the EFHT of the chain over the initial distribution $\pi_0$. Thus, the expected running time of the ($\mu$+1)-EA starting from $\xi_0 \sim \pi_0$ is $\mu+(\mu+1)\cdot \mathrm{E}(\tau \mid \xi_{0} \sim \pi_0)$, where the first $\mu$ is the cost of evaluating the initial population, and $(\mu+1)$ is the cost of one iteration, where it needs to evaluate the offspring solution $x'$ and reevaluate the $\mu$ parent solutions. Similarly, the expected running time of the (1+$\lambda$)-EA starting from $\xi_0 \sim \pi_0$ is $1+(1+\lambda)\cdot \mathrm{E}(\tau \mid \xi_{0} \sim \pi_0)$, where the first $1$ is the cost of evaluating the initial solution, and $(1+\lambda)$ is the cost of one iteration, where it needs to evaluate the $\lambda$ offspring solutions and reevaluate the parent solution. For the (1+1)-EA, the expected running time is calculated by setting $\mu=1$ or $\lambda=1$, i.e., $1+2\cdot \mathrm{E}(\tau \mid \xi_{0} \sim \pi_0)$. For the (1+1)-EA with sampling, it becomes $m+2m\cdot \mathrm{E}(\tau \mid \xi_{0} \sim \pi_0)$, because the fitness estimation of a solution needs $m$ independent evaluations. Note that in this paper, we consider the expected running time of an EA starting from a uniform initial distribution.

Next, we introduce several drift theorems which will be used to analyze the EFHT of Markov chains in this paper. The multiplicative drift theorem (i.e., Theorem~\ref{multiplicative-drift})~\cite{doerr:etal:GECCO10} is for deriving upper bounds on the EFHT. First, a distance function $V(x)$ satisfying that $V(x \in \mathcal{X}^*)=0$ and $V(x \notin \mathcal{X}^*)>0$ needs to be designed to measure the distance of a state $x$ to the target state space $\mathcal{X}^*$. Then, we need to analyze the drift towards $\mathcal{X}^*$ in each step, i.e., $\mathbb{E}(V(\xi_t)-V(\xi_{t+1}) \mid \xi_t)$. If the drift in each step is roughly proportional to the current distance to the set of optimal populations, we can derive an upper bound on the EFHT accordingly. Note that $\ln$ denotes the natural logarithm, and we will use $\log$ to denote the binary logarithm throughout the paper.

\begin{theorem}[Multiplicative Drift~\cite{doerr:etal:GECCO10}]\label{multiplicative-drift}
Given a Markov chain $\{\xi_t\}^{+\infty}_{t=0}$ and a distance function $V$ over $\mathcal{X}$, suppose there exists $c>0$ such that for all $t \geq 0$ and $\xi_t$ with $V(\xi_t) > 0$:
$$
\mathrm{E}(V(\xi_t)-V(\xi_{t+1}) \mid \xi_t) \geq c \cdot V(\xi_t).
$$
Then it holds that $\mathrm{E}(\tau \mid \xi_0) \leq \frac{1+\ln (V(\xi_0)/V_{\min})}{c}$, where $V_{\min}$ denotes the minimum among all possible positive values of \,$V$.
\end{theorem}

The simplified negative drift theorem (i.e., Theorem~\ref{simplified-drift})~\cite{oliveto2011simplified,oliveto2011simplifiedErratum} is for proving exponential lower bounds on the EFHT of Markov chains, where $X_t$ is often represented by a mapping of $\xi_t$. From Theorem~\ref{simplified-drift}, we can see that two conditions are required: (1) a constant negative drift and (2) exponentially decaying probabilities of jumping towards or away from the target state. By building a relationship between the jumping distance and the length of the drift interval, a more general theorem, simplified negative drift with scaling~\cite{oliveto2014runtime}, as presented in Theorem~\ref{simplified-drift-scaling} has been proposed. Theorem~\ref{negative-drift} gives the original negative drift theorem~\cite{hajek1982hitting}, which is stronger because both the two simplified versions are proved by using this original theorem.

\begin{theorem}[Simplified Negative Drift~\cite{oliveto2011simplified,oliveto2011simplifiedErratum}]\label{simplified-drift}
Let $X_t$, $t\geq0$, be real-valued random variables describing a stochastic process over some state space. Suppose there exists an interval $[a,b] \subseteq \mathbb{R}$, two constants $\delta,\epsilon>0$ and, possibly depending on $l:=b-a$, a function $r(l)$ satisfying $1\leq r(l)=o(l/\log(l))$ such that for all $t\geq 0$:
\begin{align}
&(1) \quad \mathrm{E}(X_t-X_{t+1} \mid a < X_t <b) \leq -\epsilon,\\
&(2) \quad \forall j \in \mathbb{N}^+: \mathrm{P}(|X_{t+1}-X_t| \geq j \mid X_t>a) \leq \frac{r(l)}{(1+\delta)^j}.
\end{align}
Then there exists a constant $c>0$ such that for $T:=\min\{t \geq 0: X_t \leq a \mid X_0 \geq b\}$ it holds $\mathrm{P}(T \leq 2^{cl/r(l)})=2^{-\Omega(l/r(l))}$.
\end{theorem}

\begin{theorem}[Simplified Negative Drift with Scaling~\cite{oliveto2014runtime}]\label{simplified-drift-scaling}
Let $X_t$, $t\geq0$, be real-valued random variables describing a stochastic process over some state space. Suppose there exists an interval $[a,b] \subseteq \mathbb{R}$ and, possibly depending on $l:=b-a$, a drift bound $ \epsilon :=\epsilon(l)>0 $ as well as a scaling factor $ r:=r(l) $ such that for all $t\geq 0$:
\begin{align}
&(1) \quad \mathrm{E}(X_t-X_{t+1} \mid a < X_t <b) \leq -\epsilon,\\
&(2) \quad \forall j \in \mathbb{N}^+: \mathrm{P}(|X_{t+1}-X_t| \geq jr \mid X_t>a) \le e^{-j},\\
&(3) \quad 1\le r\le \min\{\epsilon^2 l,\sqrt{\epsilon l/(132 \ln(\epsilon l))}\}.
\end{align}
Then it holds for the first hitting time $T:=\min\{t \geq 0: X_t \leq a \mid X_0 \geq b\}$ that $\mathrm{P}(T \le e^{\epsilon l/(132r^2)})=O(e^{-\epsilon l/(132r^2)})$.
\end{theorem}

\begin{theorem}[Negative Drift~\cite{hajek1982hitting}]\label{negative-drift}
Let $X_t,t\ge 0$, be real-valued random variables describing a stochastic process over some state space. Pick two real numbers $ a(l) $ and $ b(l) $ depending on a parameter $ l $ such that $ a(l)<b(l) $ holds. Let $ T(l) $ be the random variable denoting the earliest time $ t\ge 0 $ such that $ X_t\le a(l) $ holds. Suppose there exists $ \lambda(l)>0 $ and $ p(l)>0 $ such that for all $ t\ge 0 $:
\begin{align}\label{Hajek-cond1}
\mathrm{E}\left(e^{-\lambda(l)\cdot (X_{t+1}-X_t)}\mid a(l)<X_t<b(l)\right)\le 1-\frac{1}{p(l)}.
\end{align}
Then it holds that for all time bounds $ L(l)\ge 0$,
\begin{align}\label{Hajek-D}
&\mathrm{P}\left(T(l)\le L(l) \mid X_0\ge b(l)\right)\le e^{-\lambda(l)\cdot (b(l)-a(l))}\cdot L(l)\cdot D(l)\cdot p(l),
\end{align}
where $ D(l)=\max\left\{1,\mathrm{E}\left(e^{-\lambda(l)\cdot (X_{t+1}-b(l))} \mid X_t\ge b(l)\right)\right\} $.
\end{theorem}

\section{Populations Can Work on Some Tasks Where Sampling Fails}\label{sec-population}

Previous works~\cite{giessen2014robustness,qian2018noise,qian2016sampling} have shown that both using populations and sampling can bring robustness against noise. For example, for the OneMax problem under one-bit noise with $p=\omega((\log n)/n)$, the (1+1)-EA needs super-polynomial expected time to find the optimum~\cite{droste2004analysis}, while using a parent population size $\mu\geq 12(\ln(15n))/p$~\cite{giessen2014robustness}, an offspring population size $\lambda \geq \max\{12/p,24\}n\ln n$~\cite{giessen2014robustness} or a sample size $m=4n^3$~\cite{qian2018noise} can all reduce the expected running time to polynomial. Then, a natural question is whether there exist cases where only one of these two strategies (i.e., populations and sampling) is effective. This question has been partially addressed. For the OneMax problem under additive Gaussian noise with large variances, it has been shown that the ($\mu$+1)-EA with $\mu=\omega(1)$ needs super-polynomial time to find the optimum~\cite{friedrich2015benefit}, while the (1+1)-EA using sampling can find the optimum in polynomial time~\cite{qian2016sampling}. Now, we try to solve the other part of this question. That is, we want to prove that using populations can be better than using sampling.

For this purpose, we construct a family of artificial noisy problems. We consider the OneMax problem under symmetric noise. As presented in Definition~\ref{def_onemax}, the goal of the OneMax problem is to maximize the number of 1-bits, and the optimal solution is the string with all 1s (denoted as $1^n$). As presented in Definition~\ref{def-symm-noise}, symmetric noise returns a false fitness $C-f(x)$ with probability $1/2$. It is easy to see that under this noise model, the distribution of $f^{\mathrm{n}}(x)$ for any $x$ is symmetric about $C/2$. Note that a concrete noisy problem depends on the value of $C$.

\begin{definition}[OneMax]\label{def_onemax}
    The OneMax Problem is to find a binary string $x^* \in \{0,1\}^n$ that maximises
    $$
        f(x)=\sum\nolimits^{n}_{i=1} x_i.
    $$
\end{definition}

\begin{definition}[Symmetric Noise]\label{def-symm-noise}
Given a parameter $C \in \mathbb{R}$, let $f^{\mathrm{n}}(x)$ and $f(x)$ denote the noisy and true fitness of a solution $x$, respectively, then
\begin{align*}
f^{\mathrm{n}}(x)=\begin{cases}
f(x) & \text{with probability $1/2$},\\
C-f(x) & \text{with probability $1/2$}.
\end{cases}
\end{align*}
\end{definition}

Theorem~\ref{theo-sample-2} shows that the expected running time of the (1+1)-EA using sampling with any sample size $m$ is exponential. From the proof, we can find the reason why using sampling fails. Under symmetric noise, the distribution of $f^{\mathrm{n}}(x)$ for any $x$ is symmetric about $C/2$. Thus, for any two solutions $x$ and $y$, the distribution of $f^{\mathrm{n}}(x)-f^{\mathrm{n}}(y)$ is symmetric about 0. By sampling, the distribution of $\hat{f}(x)-\hat{f}(y)$ is still symmetric about 0, which implies that the offspring solution will always be accepted with probability at least $1/2$ in each iteration of the (1+1)-EA. Such a behavior is analogous to random walk, and thus the optimization is inefficient.

\begin{theorem}\label{theo-sample-2}
For the (1+1)-EA solving OneMax under symmetric noise with any $C \in \mathbb{R}$, if using sampling with any sample size $m \geq 1$, the expected running time is exponential.
\end{theorem}
\begin{myproof}
Let a Markov chain $\{\xi_t\}^{+\infty}_{t=0}$ model the analyzed evolutionary process. That is, $\xi_t$ corresponds to the solution after running $t$ iterations of the (1+1)-EA. We will show that for any $t \geq 0$, the distribution of $\xi_t$ is a uniform distribution over $\{0,1\}^n$, i.e.,
\begin{align}\label{eq:mid11}
\forall x \in \{0,1\}^n: \mathrm{P}(\xi_t=x)=1/2^n.
\end{align}
For $t=0$, it trivially holds since $\xi_0$ is chosen from $\{0,1\}^n$ uniformly at random. Assume that for $t \leq i$, Eq.~(\refeq{eq:mid11}) holds. Let $\mathrm{P}_{\mathrm{mut}}(x,y)$ denote the probability that $x$ is mutated to $y$ by bit-wise mutation. For $t=i+1$, we have $\forall x \in \{0,1\}^n:$
\begin{align}
&\mathrm{P}(\xi_{i+1}=x)= \sum_{y\in \{0,1\}^n} \mathrm{P}(\xi_{i+1}=x \mid \xi_i=y) \mathrm{P}(\xi_i=y)\\
&=\frac{1}{2^n}\sum_{y\neq x} \mathrm{P}(\xi_{i+1}=x \mid \xi_i=y)+ \frac{1}{2^n} \mathrm{P}(\xi_{i+1}=x \mid \xi_i=x)\\
&=\frac{1}{2^n}\sum_{y\neq x} \mathrm{P}_{\mathrm{mut}}(y,x)\cdot \mathrm{P}(\hat{f}(x)\geq \hat{f}(y))\\
&\quad + \frac{1}{2^n}\left(\sum_{y\neq x} \mathrm{P}_{\mathrm{mut}}(x,y)\cdot \mathrm{P}(\hat{f}(y)< \hat{f}(x)) + \mathrm{P}_{\mathrm{mut}}(x,x)\cdot 1\right)\\
&=\frac{1}{2^n}\left(\sum_{y\neq x} \mathrm{P}_{\mathrm{mut}}(x,y)\cdot (\mathrm{P}(\hat{f}(x)\geq \hat{f}(y))+\mathrm{P}(\hat{f}(x)> \hat{f}(y))) + \mathrm{P}_{\mathrm{mut}}(x,x)\right)\\
&=\frac{1}{2^n}\left(\sum_{y\neq x} \mathrm{P}_{\mathrm{mut}}(x,y)\cdot 1 + \mathrm{P}_{\mathrm{mut}}(x,x)\right)=\frac{1}{2^n},
\end{align}
where the second equality is by induction, i.e., $\forall x \in \{0,1\}^n: \mathrm{P}(\xi_i=x)=1/2^n$, the third equality is by considering the mutation and selection behaviors, the fourth equality is by $\mathrm{P}_{\mathrm{mut}}(y,x)=\mathrm{P}_{\mathrm{mut}}(x,y)$, and the fifth is by $\mathrm{P}(\hat{f}(x)> \hat{f}(y))=\mathrm{P}(\hat{f}(x)< \hat{f}(y))$ since $\hat{f}(x)-\hat{f}(y)$ is symmetric about 0. By the definition of symmetric noise, the value of $f^{\mathrm{n}}(x)-f^{\mathrm{n}}(y)$ can be $C-f(x)-f(y)$, $f(x)-f(y)$, $f(y)-f(x)$ and $f(x)+f(y)-C$, each with probability $1/4$. It is easy to see that the distribution of $f^{\mathrm{n}}(x)-f^{\mathrm{n}}(y)$ is symmetric about 0, i.e., $f^{\mathrm{n}}(x)-f^{\mathrm{n}}(y)$ has the same distribution as $f^{\mathrm{n}}(y)-f^{\mathrm{n}}(x)$. Since $\hat{f}(x)-\hat{f}(y)$ is the average of $m$ independent random variables, which have the same distribution as $f^{\mathrm{n}}(x)-f^{\mathrm{n}}(y)$, the distribution of $\hat{f}(x)-\hat{f}(y)$ is also symmetric about 0.

By the union bound, the probability of finding the optimum in $o(2^n)$ iterations is at most $\sum^{o(2^n)}_{t=0} \mathrm{P}(\xi_t=1^n)=o(2^n)/2^n=o(1)$. Thus, the expected running time is exponential.\vspace{0.8em}
\end{myproof}

\subsection{Parent Populations}\label{sec-parent}

In this subsection, we show that compared with using sampling, using parent populations can be more robust to noise. We prove in Theorem~\ref{theo-population-1} that for symmetric noise with $C=2n$, the ($\mu$+1)-EA with $\mu=3\log n$ can find the optimum in $O(n\log^3n)$ time. The reason for the effectiveness of using parent populations is that the true best solution will be discarded only if it appears worse than all the other solutions in the population, the probability of which can be very small by using at least a logarithmic parent population size. Note that this finding is consistent with that in~\cite{giessen2014robustness}.

\begin{theorem}\label{theo-population-1}
For the ($\mu$+1)-EA solving OneMax under symmetric noise with $C=2n$, if $\mu=3\log n$, the expected running time is $O(n\log^3n)$.
\end{theorem}
\begin{myproof}
We apply the multiplicative drift theorem (i.e., Theorem~\ref{multiplicative-drift}) to prove this result. Note that the state of the corresponding Markov chain is currently a population, i.e., a set of $\mu$ solutions. We first design a distance function $V$: for any population $P$, $V(P)=\min\nolimits_{x\in P}|x|_0$, i.e., the minimum number of 0-bits of the solution in $P$. It is easy to see that $V(P)=0$ iff $P \in \mathcal{X}^*$, i.e., $P$ contains the optimum $1^n$.

Next we examine $\mathrm{E}(V(\xi_t)-V(\xi_{t+1}) \mid \xi_t=P)$ for any $P$ with $V(P)>0$ (i.e., $P \notin \mathcal{X^*}$). Assume that currently $V(P)=i$, where $1 \leq i \leq n$. We divide the drift into two parts:
\begin{align}
&\mathrm{E}(V(\xi_t)-V(\xi_{t+1}) \mid \xi_t=P)=\mathrm{E}^+-\mathrm{E}^-, \quad \text{where}\\
&\mathrm{E}^+=\sum_{P': V(P')<i}\;\mathrm{P}(\xi_{t+1}=P'\mid \xi_t=P)\cdot (i-V(P')),\\
&\mathrm{E}^-=\sum_{P': V(P')>i}\;\mathrm{P}(\xi_{t+1}=P'\mid \xi_t=P)\cdot (V(P')-i).
\end{align}
For $\mathrm{E}^+$, we need to consider that the best solution in $P$ is improved. Let $x^*\in \arg\min_{x\in P}|x|_0$, then $|x^*|_0=i$. In one iteration of the ($\mu$+1)-EA, a solution $x'$ with $|x'|_0=i-1$ can be generated by selecting $x^*$ and flipping only one 0-bit in mutation, whose probability is $\frac{1}{\mu}\cdot \frac{i}{n}(1-\frac{1}{n})^{n-1}\geq \frac{i}{e\mu n}$. If $x'$ is not added into $P$, it must hold that $f^{\mathrm{n}}(x') <f^{\mathrm{n}}(x)$ for all $x\in P$, which happens with probability $1/2^\mu$ since $f^{\mathrm{n}}(x') <f^{\mathrm{n}}(x)$ iff $f^{\mathrm{n}}(x)=2n-f(x)$. Thus, the probability that $x'$ is added into $P$ (which implies that $V(P')=i-1$) is $1-1/2^\mu$. We then get
\begin{align}\label{mu+1 positive}
\mathrm{E}^+\ge \frac{i}{e\mu n}\cdot \left(1-\frac{1}{2^\mu}\right)\cdot (i-(i-1))=\frac{i}{e\mu n} \left(1-\frac{1}{2^\mu}\right).
\end{align}
For $\mathrm{E}^-$, if there are at least two solutions $x,y$ in $P$ such that $|x|_0=|y|_0=i$, it obviously holds that $\mathrm{E}^-=0$. Otherwise, $V(P')>V(P)=i$ implies that for the unique best solution $x^*$ in $P$ and any $x \in P\setminus\{x^*\}$, $f^{\mathrm{n}}(x^*) \leq f^{\mathrm{n}}(x)$, which happens with probability $1/2^{\mu-1}$ since $f^{\mathrm{n}}(x^*) \leq f^{\mathrm{n}}(x)$ iff $f^{\mathrm{n}}(x)=2n-f(x)$. Thus, $\mathrm{P}(V(P')>i) \le 1/2^{\mu-1}$. Furthermore, $V(P')$ can increase by at most $n-i$. Thus, $\mathrm{E}^-\le (n-i)/2^{\mu-1}$. By calculating $\mathrm{E}^+-\mathrm{E}^-$, we get
\begin{align}\label{mu+1 drift}
\mathrm{E}(V(\xi_t)-V(\xi_{t+1}) \mid V(\xi_t)=i)&\ge \frac{i}{e\mu n}-\frac{i}{e\mu n2^\mu}-\frac{n-i}{2^{\mu-1}}\\
&\ge \frac{i}{10n\log n}=\frac{1}{10n\log n} \cdot V(\xi_t),
\end{align}
where the second inequality holds with large enough $\mu$ (which depends monotonically on $n$). Note that $\mu=3\log n$. Thus, by Theorem~\ref{multiplicative-drift}, $$\mathrm{E}(\tau\mid \xi_0)\le 10n(\log n)(1+\ln n) =O(n\log^2n),$$ which implies that the expected running time is $O(n\log^3 n)$, since the algorithm needs to evaluate the offspring solution and reevaluate the $\mu$ parent solutions in each iteration.\vspace{0.8em}
\end{myproof}

In the following, we show that the parent population size $\mu=3\log n$ is almost tight for making the ($\mu$+1)-EA efficient. Particularly, we prove that $\mu\le \sqrt{\log n}/2$ is insufficient. Note that the proof is finished by applying the original negative drift theorem (i.e., Theorem~\ref{negative-drift}) instead of the simplified versions (i.e., Theorems~\ref{simplified-drift} and~\ref{simplified-drift-scaling}). To apply the simplified negative drift theorems, we have to show that the probability of jumping towards and away from the target is exponentially decaying. However, the probability of jumping away from the target is $\omega(1/n)$ in this studied case. To jump away from the target, it is sufficient that one non-best solution in the current population is cloned by mutation and then the best solution is deleted in the process of updating the population. The former event happens with probability $\frac{\mu-1}{\mu}\cdot (1-\frac{1}{n})^{n}=\Theta(1)$, and the latter happens with probability $\frac{1}{2^\mu}$, which is $\omega(1/n)$ for $\mu\le \sqrt{\log n}/2$. The original negative drift theorem is stronger than the simplified ones, and can be applied here to prove the exponential running time.

\begin{theorem}\label{theo-parent-population-2}
For the ($\mu$+1)-EA solving OneMax under symmetric noise with $C=2n$, if $\mu\le \sqrt{\log n}/2$, the expected running time is exponential.
\end{theorem}
\begin{myproof}
We apply the original negative drift theorem (i.e., Theorem~\ref{negative-drift}) to prove this result. Let $ X_t=Y_t-h(Z_t)$, where $Y_t=\min\nolimits_{x\in P}|x|_0 $ denotes the minimum number of 0-bits of the solution in the population $ P $ after $ t $ iterations of the ($\mu$+1)-EA, $ Z_t=|\{x\in P\mid |x|_0=Y_t\} |$ denotes the number of solutions in $ P $ that have the minimum 0-bits $Y_t$, and for $i \in \{1,2,\ldots,\mu\}$, $ h(i)= \frac{d^{\mu-1}-d^{\mu-i}}{d^\mu-1}$ with $d=2^{\mu+4}$. Note that $ 0=h(1)<h(2)<...<h(\mu)<1 $, and $ X_t\leq 0 $ iff $Y_t=0$, i.e., $ P $ contains at least one optimum $ 1^n $. We set $ l=n$, $ \lambda(l)=1 $ and consider the interval $[0,cn-1]$, where $ c=\frac{1}{3d^\mu} $, i.e., the parameters $a(l)=0$ and $b(l)=cn-1$ in Theorem~\ref{negative-drift}.
	
We analyze Eq.~\eqref{Hajek-cond1}, which is equivalent to the following equation:
\begin{align}\label{Hajek-cond2}
	\sum_{r\neq X_t}\mathrm{P}\left(X_{t+1}=r \mid a(l)<X_t<b(l)\right)\cdot \left(e^{X_t-r}-1\right) \le -\frac{1}{p(l)}.
\end{align}
We divide the left-side term of Eq.~\eqref{Hajek-cond2} into two parts: $ r<X_t $ (i.e., $ X_{t+1}<X_t $) and $ r>X_t $ (i.e., $ X_{t+1}>X_t $), and derive their upper bounds separately.

We first consider $ X_{t+1}<X_t $. Since $ X_{t+1}=Y_{t+1}-h(Z_{t+1})$, $X_t=Y_t-h(Z_t) $ and $0\le h(Z_{t+1}),h(Z_t)<1$, we have $X_{t+1}<X_t$ iff $ Y_{t+1}-Y_t<0 $ or $ Y_{t+1}=Y_t \wedge h(Z_{t+1})>h(Z_t) $. In the following, we analyze the occurring probability of each case, and the corresponding value of $X_{t}-X_{t+1}$.\\
(1) $ Y_{t+1}-Y_t=-j\le -1$. It implies that a new solution $ x' $ with $|x'|_0= Y_t-j $ is generated in the $(t+1)$-th iteration of the algorithm. Suppose that $ x' $ is generated from some solution $ x $ (which must satisfy that $ |x|_0\ge Y_t $) selected from $P$, then
\begin{align}
\sum_{x': |x'|_0= Y_t-j}\mathrm{P_{mut}}(x,x')&\le \sum_{x': |x'|_0= Y_t-j}\mathrm{P_{mut}}\left(x^{Y_t},x'\right)\\
&\le \binom{Y_t}{j}\cdot \frac{1}{n^j}\le \left(\frac{Y_t}{n}\right)^j< c^j,
\end{align}
where $ x^j $ denotes any solution with $j$ 0-bits, the second inequality is because it is necessary to flip at least $j$ 0-bits, and the last inequality is by $ Y_t=X_t+h(Z_t)<b(l)+1=cn $. Furthermore, we have
$$ X_t-X_{t+1}=Y_{t}-h(Z_{t})-Y_{t+1}+h(Z_{t+1})=j-h(Z_t)\leq j, $$
where the second equality is by $ h(Z_{t+1})=h(1)=0$. \\
(2) $ Y_{t+1}=Y_t \wedge h(Z_{t+1})>h(Z_t) $. It implies that $ Z_t<\mu $ and a new solution $ x' $ with $ |x'|_0=Y_t $ is generated. Suppose that in the $(t+1)$-th iteration, the solution selected from $P$ for mutation is $x$. If $ |x|_0>Y_t $, then $ \sum_{x':|x'|_0=Y_t}\mathrm{P_{mut}}(x,x')\le \sum_{x':|x'|_0=Y_t}\mathrm{P_{mut}}(x^{Y_t+1},x')\le \binom{Y_t+1}{1}\cdot \frac{1}{n} = \frac{Y_t+1}{n}$. If $ |x|_0=Y_t $, then $ \sum_{x':|x'|_0=Y_t}\mathrm{P_{mut}}(x,x')\le (1-\frac{1}{n})^n+\sum_{j=1}^{Y_t}\binom{Y_t}{j}\cdot \frac{1}{n^j} \le \frac{1}{e}+\sum_{j=1}^{Y_t}(\frac{Y_t}{n})^j \le \frac{1}{e}+\frac{Y_t/n}{1-Y_t/n}$. Since $Y_t=X_t+h(Z_t)<b(l)+1=cn$ and $c=\frac{1}{3d^\mu}=\frac{1}{3\cdot 2^{\mu(\mu+4)}}$, we have $$ \sum_{x':|x'|_0=Y_t}\mathrm{P_{mut}}(x,x')\le \frac{1}{2}. $$
Furthermore, it must hold that $Z_{t+1}=Z_t+1$, thus we have
$$ X_t-X_{t+1}=h(Z_{t+1})-h(Z_t)=h(Z_t+1)-h(Z_t) .$$
By combining the above cases, we get
\begin{align}\label{Hajek-E+}
& \sum_{r< X_t}\mathrm{P}\left(X_{t+1}=r \mid a(l)<X_t<b(l)\right)\cdot \left(e^{X_t-r}-1\right)\\
&\le \sum_{j=1}^{Y_t}c^j\cdot \left(e^{j}-1\right)+ \left\{\begin{array}{ll}
\frac{1}{2}\cdot \left(e^{h(Z_t+1)-h(Z_t)}-1\right), & \;\; Z_t<\mu\\
0, & \;\;Z_t=\mu
\end{array}\right. \\
&\le  \sum_{j=1}^{Y_t}(ce)^j+\left\{\begin{array}{ll}
h(Z_t+1)-h(Z_t), & \;\; Z_t<\mu\\
0, & \;\; Z_t=\mu
\end{array}\right. \\
&\le  \frac{ce}{1-ce}+\left\{\begin{array}{ll}
h(Z_t+1)-h(Z_t), & \;\; Z_t<\mu\\
0, & \;\; Z_t=\mu
\end{array}\right.,
\end{align}
where the second inequality is by $0< h(Z_t+1)-h(Z_t)<1$ and $ e^s\le 1+2s $ for $ 0<s<1 $.

Next we consider $ X_{t+1}>X_t $. It is easy to verify that $ X_{t+1}>X_t $ iff in the $(t+1)$-th iteration, the newly generated solution $x'$ satisfies that $ |x'|_0>Y_t $ and one solution $ x^* $ in $P$ with $ |x^*|_0=Y_t $ is deleted. We first analyze the probability of generating a new solution $x'$ with $ |x'|_0>Y_t $. Suppose that the solution selected from $P$ for mutation is $x$. If $ |x|_0>Y_t $, it is sufficient that all bits of $ x $ are not flipped, thus $ \sum_{x':|x'|_0>Y_t}\mathrm{P_{mut}}(x,x')\ge (1-\frac{1}{n})^n\ge \frac{n-1}{en}$. If $ |x|_0=Y_t $, it is sufficient that only one 1-bit of $ x $ is flipped, thus $ \sum_{x':|x'|_0>Y_t}\mathrm{P_{mut}}(x,x')\ge (1-\frac{1}{n})^{n-1}\frac{n-Y_t}{n}\ge \frac{n-Y_t}{en}$. Note that $Y_t=X_t+h(Z_t)<b(l)+1=cn$ and $c=\frac{1}{3\cdot 2^{\mu(\mu+4)}}=\omega(1/n)$ for $\mu\leq  \sqrt{\log n}/2$. Thus,
$$ \sum_{x':|x'|_0>Y_t}\mathrm{P_{mut}}(x,x')\ge \frac{1-c}{e}. $$
We then analyze the probability of deleting one solution $x^*$ in $P$ with $|x^*|_0=Y_t$. Since it is sufficient that the fitness evaluation of all solutions in $P \cup \{x'\}$ with more than $Y_t$ 0-bits is affected by noise, the probability is at least $ 1/2^\mu $. We finally analyze $ X_{t}-X_{t+1}$. If $ Z_t=1 $, we have $ Y_{t+1}\ge Y_t+1 $, thus
$$ X_t-X_{t+1}=Y_{t}-Y_{t+1}+h(Z_{t+1})-h(Z_{t})\leq h(\mu)-1.$$
If $ Z_t\ge 2 $, we have $ Y_{t+1}=Y_t $ and $ Z_{t+1}=Z_t-1 $, thus
$$X_{t}-X_{t+1} =h(Z_{t+1})-h(Z_t)=h(Z_t-1)-h(Z_t).$$
Note that for $X_{t+1}>X_t$, $e^{X_t-X_{t+1}}-1<0$. Thus, we have
\begin{align}\label{Hajek-E-}
& \sum_{r> X_t}\mathrm{P}\left(X_{t+1}=r\mid a(l)<X_t<b(l)\right)\cdot \left(e^{X_t-r}-1\right)\\
&\le  \frac{1}{2^\mu}\cdot \frac{1-c}{e}\cdot \left\{
\begin{array}{ll}
e^{h(\mu)-1}-1, &\;\; Z_t=1\\
e^{h(Z_t-1)-h(Z_t)}-1, & \;\;Z_t\ge 2\\
\end{array}\right.\\
&\le \frac{1}{2^{\mu+1}}\cdot \frac{1-c}{e}\cdot \left\{
\begin{array}{ll}
h(\mu)-1, & \;\;Z_t=1\\
h(Z_t-1)-h(Z_t), &\;\; Z_t\ge 2\\
\end{array}\right.\\
&\le \frac{2}{d}\cdot \left\{
\begin{array}{ll}
h(\mu)-1, &\;\; Z_t=1\\
h(Z_t-1)-h(Z_t), & \;\;Z_t\ge 2\\
\end{array}\right.,
\end{align}
where the second inequality is by $ e^s-1\le s+s^2/2=s(1+s/2)\le s/2$ for $ -1<s<0 $, and the last is by $d=2^{\mu+4}$ and $c=\frac{1}{3\cdot 2^{\mu(\mu+4)}}$.

By combining Eq.~\eqref{Hajek-E+} and Eq.~\eqref{Hajek-E-}, we can get
\begin{align}
&\sum_{r\neq X_t}\mathrm{P}\left(X_{t+1}=r\mid a(l)<X_t<b(l)\right)\cdot \left(e^{X_t-r}-1\right)\\
&\le   \frac{ce}{1-ce}+\left\{\begin{array}{ll}
h(Z_t+1)-h(Z_t)+\frac{2}{d}(h(\mu)-1), &\;\; Z_t=1\\
h(Z_t+1)-h(Z_t)+\frac{2}{d}(h(Z_t-1)-h(Z_t)), &\;\; 1<Z_t<\mu\\
\frac{2}{d}(h(Z_t-1)-h(Z_t)), &\;\; Z_t=\mu
\end{array}\right..
\end{align}
If $ Z_t=1 $, $ \frac{1-h(\mu)}{h(Z_t+1)-h(Z_t)} =\frac{d^\mu-d^{\mu-1}}{d^\mu-1}\cdot \frac{d^\mu-1}{d^{\mu-1}-d^{\mu-2}}=d$,
and we have $ h(Z_t+1)-h(Z_t)+\frac{2}{d}(h(\mu)-1) =(h(Z_t+1)-h(Z_t))\cdot (1-d\cdot \frac{2}{d})\le h(\mu-1)-h(\mu) $. If $ 1< Z_t<\mu $, $ \frac{h(Z_t)-h(Z_t-1)}{h(Z_t+1)-h(Z_t)}=\frac{d^{\mu-Z_t+1}-d^{\mu-Z_t}}{d^{\mu-Z_t}-d^{\mu-Z_t-1}}=d $,
and similarly we have $ h(Z_t+1)-h(Z_t)+\frac{2}{d}(h(Z_t-1)-h(Z_t)) = h(Z_t)-h(Z_t+1)\le h(\mu-1)-h(\mu)$.
If $ Z_t=\mu $, $\frac{2}{d}(h(Z_t-1)-h(Z_t))= \frac{2}{d}(h(\mu-1)-h(\mu)) $. Thus, the above equation continues with
\begin{align}
&\le \frac{ce}{1-ce}+\frac{2}{d}(h(\mu-1)-h(\mu)) = \frac{1}{1/(ce)-1}+\frac{2}{d}\cdot \frac{1-d}{d^\mu-1}\\
&\le \frac{1}{d^\mu-1}-\frac{3}{2}\cdot \frac{1}{d^\mu-1}
= -\frac{1}{2(d^\mu-1)},
\end{align}
where the second inequality is by $ c=\frac{1}{3d^\mu} $ and $ d\ge 4 $. The condition of Theorem~\ref{negative-drift} (i.e., Eq.~\eqref{Hajek-cond1} or equivalently Eq.~\eqref{Hajek-cond2}) thus holds with $ p(l)=2(d^\mu-1) $.

Now we investigate $ D(l)=\max\left\{1,\mathrm{E}\left(e^{-\lambda(l)\cdot (X_{t+1}-b(l))} \mid X_t\ge b(l)\right)\right\}=  \max\left\{1,\mathrm{E}\left(e^{ b(l)-X_{t+1}} \mid X_t\ge b(l)\right)\right\}$ in Eq.~\eqref{Hajek-D}. To derive an upper bond on $ D(l) $, we only need to analyze $ \mathrm{E}\left(e^{b(l)-X_{t+1}}\mid X_t\ge b(l)\right) $.
\begin{align}
&\mathrm{E}\left(e^{b(l)-X_{t+1}}\mid X_t\ge b(l)\right)\\
&=\sum_{r\ge b(l)}\mathrm{P}(Y_{t+1}=r \mid X_t\ge b(l)) \cdot \mathrm{E}\left(e^{b(l)-X_{t+1}}\mid X_t\ge b(l), Y_{t+1}=r\right)\\
&\quad +\sum_{r< b(l)}\mathrm{P}(Y_{t+1}=r\mid X_t\ge b(l))\cdot \mathrm{E}\left(e^{b(l)-X_{t+1}}\mid X_t\ge b(l), Y_{t+1}=r\right).
\end{align}
When $Y_{t+1}=r \geq b(l)$, we have $ b(l)-X_{t+1}=b(l)-Y_{t+1}+h(Z_{t+1})\le h(Z_{t+1}) < 1$. Next we consider the case that $ Y_{t+1}<b(l) $. Since $ X_t=Y_t-h(Z_t)\ge b(l) $, we have $ Y_t\ge b(l)>Y_{t+1} $, which implies that $ Y_t\ge \lceil b(l)\rceil$ and $Y_{t+1}\le \lceil b(l)\rceil-1 $. To make $Y_{t+1}=r\leq \lceil b(l)\rceil-1$, it is necessary that a new solution $x'$ with $|x'|_0=r \leq \lceil b(l)\rceil-1$ is generated by mutation. Let $x$ denote the solution selected from the population $P$ for mutation. Note that $|x|_0 \geq Y_t \geq \lceil b(l)\rceil$. Then, for $r\leq \lceil b(l)\rceil-1$, $\mathrm{P}(Y_{t+1}=r\mid X_t\ge b(l)) \leq \sum_{x':|x'|_0=r}\mathrm{P_{mut}}(x,x')\le \sum_{x':|x'|_0=r}\mathrm{P_{mut}}(x^{\lceil b(l)\rceil},x')\leq \binom{\lceil b(l)\rceil}{\lceil b(l)\rceil-r}(\frac{1}{n})^{\lceil b(l)\rceil-r}\le (\frac{\lceil b(l)\rceil}{n})^{\lceil b(l)\rceil-r}$. Furthermore, for $Y_{t+1}<Y_t$, it must hold that $Z_{t+1}=1$, and thus $b(l)-X_{t+1}=b(l)-Y_{t+1}+h(Z_{t+1})=b(l)-Y_{t+1}$. Thus, the above equation continues with
\begin{align}
&\le  e+\sum_{r\le \lceil b(l)\rceil-1}\left(\frac{\lceil b(l)\rceil}{n}\right)^{\lceil b(l)\rceil-r}\cdot e^{b(l)-r}\leq  e+\sum_{j=1}^{\lceil b(l)\rceil}\left(\frac{\lceil b(l)\rceil}{n}\right)^j\cdot e^{j} \\
&\le e+\frac{e\lceil b(l)\rceil/n}{1-e\lceil b(l)\rceil/n}=  e+\frac{1}{n/(e\lceil b(l)\rceil)-1} \le e+\frac{1}{1/(ce)-1} \le e+1,
\end{align}
where the fourth inequality is by $ \lceil b(l)\rceil\le b(l)+1=cn $ and the last inequality is by $ c=\frac{1}{3d^\mu} $. Thus,
$$  D(l)=\max\left\{1,\mathrm{E}\left(e^{b(l)-X_{t+1}}\mid X_t\ge b(l)\right)\right\}\le e+1. $$

Let $ L(l)=e^{cn/2} $ in Theorem~\ref{negative-drift}. As $\mu\le \sqrt{\log n}/2$, we have
$$3d^\mu= 3\cdot 2^{\mu(\mu+4)}\le 3\cdot 2^{(\log n)/4+2\sqrt{\log n}} \le 2^{(\log n)/2}=n^{1/2},$$
where the last inequality holds with large enough $n$. Thus, $ cn=\frac{n}{3d^\mu}\ge n^{1/2} $. By Theorem~\ref{negative-drift}, we get
\begin{align}
&\mathrm{P}(T(l)\le e^{cn/2} \mid X_0\ge b(l))\le e^{1-cn}\cdot e^{cn/2}\cdot (e+1)\cdot 2(d^\mu-1)=e^{-\Omega(n^{1/2})}.
\end{align}
By Chernoff bounds, for any $ x $ chosen from $ \{0,1\}^n $ u.a.r., $ \mathrm{P}(|x|_0< cn)=e^{-\Omega(n)}$, where $cn=\frac{n}{3d^\mu}=\frac{n}{3\cdot 2^{\mu(\mu+4)}} \leq \frac{n}{96}$. By the union bound, $ \mathrm{P}(Y_0< cn) \le \mu\cdot e^{-\Omega(n)}=e^{-\Omega(n)}$, which implies that $ \mathrm{P}(X_0<b(l))=\mathrm{P}(Y_0-h(Z_0)<b(l))\le \mathrm{P}(Y_0<b(l)+1)=\mathrm{P}(Y_0<cn) =e^{-\Omega(n)} $. Thus, the expected running time is exponential.
\end{myproof}

\subsection{Offspring Populations}\label{sec-offspring}

Next, we show the superiority of using offspring populations over sampling on the robustness to noise. We prove in Theorem~\ref{theo-population-offpsring} that for symmetric noise with $C=0$, the (1+$\lambda$)-EA with $\lambda=8\log n$ can find the optimum in $O(n\log^2n)$ time. By using offspring populations, the probability of losing the current fitness becomes very small. This is because a fair number of offspring solutions with fitness not worse than the current fitness will be generated with a high probability in the reproduction of each iteration of the (1+$\lambda$)-EA, and the current fitness becomes worse only if all these good offspring solutions and the parent solution are evaluated incorrectly, the probability of which can be very small by using at least a logarithmic offspring population size. Thus, using offspring populations can lead to an efficient optimization. Note that the reason for the effectiveness of using offspring populations found here is consistent with that in~\cite{giessen2014robustness}.

\begin{theorem}\label{theo-population-offpsring}
For the (1+$ \lambda $)-EA solving OneMax under symmetric noise with $C=0$, if $\lambda=8\log n$, the expected running time is $O(n\log^2 n)$.
\end{theorem}
\begin{myproof}
We apply Theorem~\ref{multiplicative-drift} to prove this result. Each state of the corresponding Markov chain $\{\xi_t\}^{+\infty}_{t=0}$ is just a solution here. That is, $\xi_t$ corresponds to the solution after running $t$ iterations of the (1+$\lambda$)-EA. We design the distance function as for $x \in \{0,1\}^n$, $ V(x)=|x|_0$. Assume that currently $|x|_0=i$, where $1 \leq i \leq n$. To analyze $\mathrm{E}(V(\xi_t)-V(\xi_{t+1}) \mid \xi_t=x)$, we divide it into two parts as in the proof of Theorem~\ref{theo-population-1}. That is,
\begin{align}
&\mathrm{E}(V(\xi_t)-V(\xi_{t+1}) \mid \xi_t=x)=\mathrm{E}^+-\mathrm{E}^-, \quad \text{where}\\
&\mathrm{E}^+=\sum_{y: |y|_0<i}\;\mathrm{P}(\xi_{t+1}=y \mid \xi_t=x)\cdot (i-|y|_0),\\
&\mathrm{E}^-=\sum_{y: |y|_0>i}\;\mathrm{P}(\xi_{t+1}=y \mid \xi_t=x)\cdot (|y|_0-i).
\end{align}

For $ \mathrm{E}^+ $, since $|y|_0<i$, we have $i-|y|_0\geq 1$. Thus,
\begin{align}
\mathrm{E}^+\geq \sum_{y: |y|_0<i}\mathrm{P}(\xi_{t+1}=y \mid \xi_t=x)=\mathrm{P}(|\xi_{t+1}|_0<i \mid \xi_t=x).
\end{align}
To make $|\xi_{t+1}|_0<i$, it requires that at least one solution $x'$ with $|x'|_0<i$ is generated in the reproduction and at least one of them is evaluated correctly. To generate a solution $x'$ with $|x'|_0<i$ by mutating $x$, it is sufficient that only one 0-bit of $x$ is flipped, whose probability is $\frac{i}{n}\cdot (1-\frac{1}{n})^{n-1} \ge \frac{i}{en}$. Thus, in each iteration of the (1+$\lambda$)-EA, the probability of generating at least one offspring solution $x'$ with $ |x'|_0<i$ is at least
\begin{align}
1-\left(1-\frac{i}{en}\right)^{\lambda} \ge 1-e^{-\lambda \cdot \frac{i}{en}}\ge 1-\frac{1}{1+\lambda\cdot \frac{i}{en}}.
\end{align}
If $ \lambda\cdot \frac{i}{en}> 1 $, $ 1-(1-\frac{i}{en})^{\lambda}\ge \frac{1}{2} $; otherwise, $ 1-(1-\frac{i}{en})^{\lambda}\ge \frac{\lambda\cdot \frac{i}{en}}{1+\lambda\cdot \frac{i}{en}} \ge \frac{\lambda \cdot i}{2en}$. Thus, $ 1-(1-\frac{i}{en})^{\lambda}\ge \min\{\frac{1}{2}, \frac{\lambda \cdot i}{2en}\}= \min\{\frac{1}{2}, \frac{4i\log n}{en}\} $, where the equality is by $\lambda=8\log n$. Since each solution is evaluated correctly with probability $ \frac{1}{2} $, $\mathrm{P}(|\xi_{t+1}|_0<i \mid \xi_t=x) \geq \min\{\frac{1}{2}, \frac{4i\log n}{en}\} \cdot \frac{1}{2}$. Thus,
$$
\mathrm{E}^+\ge \min\left\{\frac{1}{2}, \frac{4i\log n}{en}\right\} \cdot \frac{1}{2}=\min\left\{\frac{1}{4}, \frac{2i\log n}{en}\right\}\ge \frac{i}{4n}.
$$

For $ \mathrm{E}^-$, since $|y|_0 -i\leq n-i$, we have
\begin{align}
\mathrm{E}^-\leq (n-i)\cdot \mathrm{P}(|\xi_{t+1}|_0>i \mid \xi_t=x).
\end{align}
Let $ q=\sum_{x': |x'|_0\leq i}\mathrm{P}_{\mathrm{mut}}(x,x') $ denote the probability of generating an offspring solution $ x' $ with at most $i$ 0-bits by mutating $ x $. Since it is sufficient that no bit is flipped or only one 0-bit is flipped in mutation, $q\geq (1-\frac{1}{n})^n+\frac{i}{n}\cdot (1-\frac{1}{n})^{n-1}\ge \frac{1}{e} $. Now we analyze $\mathrm{P}(|\xi_{t+1}|_0> i \mid \xi_t=x)$. Assume that in the reproduction, exactly $ k $ offspring solutions with at most $ i $ 0-bits are generated, where $0\le k\le \lambda$; it happens with probability $\binom{\lambda}{k}\cdot q^k(1-q)^{\lambda-k}$. If $ k<\lambda $, the solution in the next generation has more than $ i $ 0-bits (i.e., $|\xi_{t+1}|_0>i$) iff the fitness evaluation of these $ k $ offspring solutions and the parent solution~$x$ are all affected by noise, whose probability is $\frac{1}{2^{k+1}} $. If $ k=\lambda $, the solution in the next generation must have at most $ i $ 0-bits (i.e., $|\xi_{t+1}|_0\leq i$). Thus, we have
\begin{align}\label{lam-work-V+}
\mathrm{P}(|\xi_{t+1}|_0> i \mid \xi_t=x)&= \sum_{k=0}^{\lambda-1}\binom{\lambda}{k}\cdot q^k(1-q)^{\lambda-k}\cdot \frac{1}{2^{k+1}} \\
&\le  \frac{1}{2} \left(1-\frac{q}{2}\right)^{\lambda}\le \frac{1}{2}\left(1-\frac{1}{2e}\right)^{\lambda},
\end{align}
where the last inequality is by $q \geq \frac{1}{e}$. We then get
\begin{align}
\mathrm{E}^-\le (n-i)\cdot \frac{1}{2}\cdot \left(1-\frac{1}{2e}\right)^{8\log n}\le \frac{n-i}{2n^{2.3}}\le \frac{1}{2n^{1.3}}.
\end{align}

By calculating $\mathrm{E^+}-\mathrm{E^-}$, we have
\begin{align}
\mathrm{E}(V(\xi_t)-V(\xi_{t+1}) \mid V(\xi_t)=i)\ge \frac{i}{4n}-\frac{1}{2n^{1.3}}\ge \frac{i}{5n}=\frac{1}{5n}\cdot V(\xi_t),
\end{align}
where the second inequality holds with large enough $ n $. Thus, by Theorem~\ref{multiplicative-drift}, $$\mathrm{E}(\tau\mid \xi_0)\le 5n(1+\ln n) =O(n\log n),$$ which implies that the expected running time is $O(n\log^2 n)$, since it needs to reevaluate the parent solution and evaluate the $\lambda=8\log n$ offspring solutions in each iteration.\vspace{0.8em}
\end{myproof}

Furthermore, we prove that an offspring population size $\lambda\le(\log n)/10$ is not sufficient to allow solving the noisy problem in polynomial time. This also implies that the effective value $\lambda=8\log n$ derived in the above theorem is nearly tight. From the proof, we can find that $\lambda\le(\log n)/10$ cannot guarantee a sufficiently small probability of losing the current fitness, and thus the optimization is inefficient.

\begin{theorem}\label{theo-offspring-population-2}
For the (1+$ \lambda $)-EA solving OneMax under symmetric noise with $C=0$, if $\lambda\le (\log n)/10$, the expected running time is exponential.
\end{theorem}
\begin{myproof}
We apply Theorem~\ref{simplified-drift-scaling} to prove this result. Let $X_t=|x|_0$ denote the number of 0-bits of the solution $x$ maintained by the (1+$\lambda$)-EA after running $t$ iterations. We consider the interval $[0,\frac{n}{16(2e)^\lambda}]$, i.e., $a=0$ and $b=\frac{n}{16(2e)^\lambda}$ in Theorem~\ref{simplified-drift-scaling}.
	
We analyze $\mathrm{E}(X_t-X_{t+1} \mid X_t=i)$ for $1\leq i<\frac{n}{16(2e)^\lambda}$. We divide the drift as follows:
\begin{align}
&\mathrm{E}(X_t-X_{t+1} \mid X_t=i)=\mathrm{E}^+-\mathrm{E}^-, \quad \text{where}\\
&\mathrm{E}^+=\sum_{j=0}^{i-1}\mathrm{P}(X_{t+1}=j \mid X_t=i)\cdot (i-j),\\
&\mathrm{E}^-=\sum_{j=i+1}^n\mathrm{P}(X_{t+1}=j \mid X_t=i)\cdot (j-i).
\end{align}
For $ \mathrm{E^+} $, we need to derive an upper bound on $ \mathrm{P}(X_{t+1}=j \mid X_t=i) $ for $ j<i $. Note that $ X_{t+1}=j $ implies that at least one offspring solution $x'$ with $ |x'|_0=j $ is generated by mutating $x$ in the reproduction. Thus, we have
\begin{align}
\mathrm{P}(X_{t+1}= j \mid X_t=i)&\le 1-\left(1-\sum_{x':|x'|_0=j}\mathrm{P_{mut}}(x,x')\right)^{\lambda}\\
&\le \lambda\cdot \sum_{x':|x'|_0=j}\mathrm{P_{mut}}(x,x'),
\end{align}
where the second inequality is by Bernoulli's inequality. Then, we get
\begin{align}\label{eq:mid12}
\mathrm{E^+}&\le \sum_{j=0}^{i-1}\lambda\cdot \left(\sum_{x':|x'|_0=j}\mathrm{P_{mut}}(x,x')\right)\cdot (i-j)\\
&=\lambda \cdot \sum\limits_{x':|x'|_0 <i} \mathrm{P}_{\mathrm{mut}}(x,x')\cdot (i-|x'|_0) \\
&=\lambda \cdot \sum^i_{k=1} k \cdot \mathrm{P}(X-Y=k)\\
&=\lambda \cdot \sum\limits^i_{k=1} k \cdot \sum\limits^i_{j=k} \mathrm{P}(X=j) \cdot \mathrm{P}(Y=j-k)\\
&=\lambda \cdot \sum\limits^i_{j=1} \sum\limits^j_{k=1} k \cdot \mathrm{P}(X=j) \cdot \mathrm{P}(Y=j-k)\\
&\leq \lambda \sum\limits^i_{j=1}  j \cdot \mathrm{P}(X=j)=\lambda\cdot \frac{i}{n},
\end{align}
where the second equality holds by letting $X$ and $Y$ denote the number of flipped 0-bits and 1-bits in mutating $x$ (where $|x|_0=i$), respectively, and the last equality holds because $X$ satisfies the binomial distribution $B(i,\frac{1}{n})$. For $\mathrm{E^-}$, we easily have
$$
\mathrm{E}^-\geq \sum_{j=i+1}^n\mathrm{P}(X_{t+1}=j \mid X_t=i)=\mathrm{P}(X_{t+1}> i \mid X_t=i).
$$
Let $q=\sum_{x': |x'|_0\leq i} \mathrm{P}_{\mathrm{mut}}(x,x')$, where $x$ is any solution with $i$ 0-bits. Using the same analysis as Eq.~(\refeq{lam-work-V+}), we can get
\begin{align}
&\mathrm{P}(X_{t+1}> i \mid X_t=i)
=\sum_{k=0}^{\lambda-1}\binom{\lambda}{k}\cdot q^k(1-q)^{\lambda-k}\cdot \frac{1}{2^{k+1}}\\
&=\frac{1}{2}\cdot \left( \left(1-\frac{q}{2}\right)^{\lambda}-\left(\frac{q}{2}\right)^{\lambda}\right)=\frac{1}{2}\cdot \left( \left(\frac{q}{2}+1-q\right)^{\lambda}-\left(\frac{q}{2}\right)^{\lambda}\right) \\
&=\frac{1}{2}\cdot \left(\sum^{\lambda}_{i=0} \binom{\lambda}{i}\left(\frac{q}{2}\right)^{\lambda-i} (1-q)^i-\left(\frac{q}{2}\right)^{\lambda}\right)\\
&\geq \frac{1}{2}\cdot \left(\left(\frac{q}{2}\right)^{\lambda}+\lambda \left(\frac{q}{2}\right)^{\lambda-1}(1-q)-\left(\frac{q}{2}\right)^{\lambda}\right)\\
&= \frac{1}{2}\cdot \lambda\left(\frac{q}{2}\right)^{\lambda-1} (1-q) \ge \lambda \cdot\frac{1}{8(2e)^\lambda},
\end{align}
where the last inequality is by $ q\ge \frac{1}{e} $ and $ 1-q\ge \sum_{x':|x'|_0=i+1}\mathrm{P_{\mathrm{mut}}}(x,x')\ge \frac{n-i}{en}\ge \frac{1}{4} $. Thus, $\mathrm{E^-}\ge \lambda/(8 (2e)^\lambda)$. By calculating $\mathrm{E^+}-\mathrm{E^-}$, we have
$$	\mathrm{E}(X_t-X_{t+1}\mid X_t=i)\le \lambda\cdot \frac{i}{n}-\lambda\cdot \frac{1}{8(2e)^\lambda}\le - \frac{\lambda}{16(2e)^\lambda},$$
where the last inequality is by $ i<\frac{n}{16(2e)^\lambda} $. Thus, condition~(1) of Theorem~\ref{simplified-drift-scaling} holds with $ \epsilon=\frac{\lambda}{16(2e)^\lambda} $.

Next we examine conditions~(2) and~(3) of Theorem~\ref{simplified-drift-scaling} by setting $r=n^{1/6}$. To make $|X_{t+1}-X_{t}|\geq jr$, it is necessary that at least one offspring solution generated by mutating $ x $ flips at least $\lfloor jr \rfloor$ bits of $x$. Let $ p(k) $ denote the probability that at least $ k $ bits of $ x $ are flipped in mutation. We easily have $p(k) \leq \binom{n}{k}\frac{1}{n^k}$. Thus,
\begin{align}\label{eq:mid10}
&\mathrm{P}(|X_{t+1}-X_{t}\mid \ge jr \mid X_t \ge 1) \le 1-(1-p(\lfloor jr \rfloor))^\lambda \\
&\le \lambda \cdot p(\lfloor jr \rfloor)
 \le \lambda\cdot \binom{n}{\lfloor jr \rfloor}\frac{1}{n^{\lfloor jr \rfloor}} \le 2\lambda \cdot \frac{1}{2^{\lfloor jr \rfloor}}\le \frac{4\lambda}{(2^{n^{1/6}})^j} \le \frac{1}{e^j},
\end{align}
where the last inequality holds with $\lambda\le (\log n)/10$ and large enough $n$. Thus, condition~(2) of Theorem~\ref{simplified-drift-scaling} holds. Since $\epsilon=\frac{\lambda}{16(2e)^\lambda}$ and $l=b-a=\frac{n}{16(2e)^\lambda}$, we have
\begin{align}
\frac{n^{1/2}}{256}\le  \epsilon l=\frac{n\lambda}{(16(2e)^\lambda)^2} \le n\log n,
\end{align}
where the first inequality is by $(2e)^\lambda\leq (2e)^{(\log n)/10}=(n^{\log(2e)})^{1/10}\le n^{1/4}$. Thus, we have
\begin{align}
\sqrt{\epsilon l/(132\ln (\epsilon l))}\ge \sqrt{n^{1/2}/(256\cdot 132\cdot2 \cdot \ln n)}\ge n^{1/6},
\end{align}
where the first inequality is by $\ln (\epsilon l)\le \ln (n\log n)\le 2\ln n$, and the second holds with large enough $n$. Furthermore, we have $\epsilon^2 l=  \frac{n\lambda^2}{(16(2e)^\lambda)^3}\ge \frac{n}{16^3n^{3/4}}\ge n^{1/6}$. Thus, $ 1\le r \le \min\{\epsilon^2 l,\sqrt{\epsilon l/(132\ln (\epsilon l))}\}$ for large enough $ n $, implying that condition~(3) of Theorem~\ref{simplified-drift-scaling} holds.

Note that $ \epsilon l/(132r^2)\ge n^{1/2}/(256\cdot 132\cdot n^{1/3})=\Omega(n^{1/6}) $ and $ X_0\ge b=\frac{n}{16(2e)^{\lambda}}$ holds with a high probability under the uniform initial distribution. By Theorem~\ref{simplified-drift-scaling}, we get that the expected running time is exponential.\vspace{0.8em}
\end{myproof}

Therefore, to reduce the expected running time from exponential to polynomial for solving the OneMax problem under symmetric noise, Theorems~\ref{theo-population-1} and~\ref{theo-parent-population-2} imply that the smallest required parent population size $\mu$ belongs to $(\sqrt{\log n}/2,3\log n]$ when $C=2n$; Theorems~\ref{theo-population-offpsring} and~\ref{theo-offspring-population-2} imply that the smallest required offspring population size $\lambda$ belongs to $((\log n)/10,8\log n]$ when $C=0$. It is challenging to find their exact values. For example, if applying the drift theorems, one needs to design a distance function to measure the distance of a population to the set of optimal populations and analyze the distance change by one step. For the ($\mu$+1)-EA, the solutions in the parent population can vary considerably, making it difficult to design a distance function measuring the quality of the parent population well. Using the minimum number of 0-bits of the solution in the population as in the proof of Theorem~\ref{theo-population-1} is probably insufficient. For estimating the one-step distance change well, one needs to compute the distribution of the offspring solution accurately, which is also difficult as there are $\mu$ parent solutions to be uniformly selected for mutation. For the (1+$\lambda$)-EA, the distance function is much easier to be designed because there is only one parent solution. However, computing the distribution of the offspring solutions is still difficult, as there are $\lambda$ offspring solutions to be independently generated.

\section{Adaptive Sampling Can Work on Some Tasks Where Both Sampling and Populations Fail}\label{sec-adaptive}

In this section, we first theoretically examine whether there exist cases where using neither populations nor sampling is effective. We give a positive answer by considering OneMax under segmented noise. Next we prove that in such a situation, using adaptive sampling can be effective, which provides some theoretical justification for the good empirical performance of adaptive sampling in practice~\cite{syberfeldt2010evolutionary,zhang2007immune}.

As presented in Definition~\ref{def-seg-noise}, the OneMax problem is divided into four segments. In one segment, the fitness is evaluated correctly, while in the other three segments, the fitness is disturbed by different noises. All seven sub-functions in Definition~\ref{def-seg-noise} are plotted in Figure~\ref{fig-segmented-noise}. Note that for the last sub-function $-n^4-\delta$ where $\delta \sim \mathcal{U}[0,1]$, we plot its expectation, i.e., a constant function $-n^4-1/2$.

\begin{definition}[OneMax under Segmented Noise]\label{def-seg-noise}
For any $x\in \{0,1\}^n$, the noisy fitness value $f^{\mathrm{n}}(x)$ is calculated as:\vspace{0.2em}\\
(1) if $|x|_0>\frac{n}{50}$, $f^{\mathrm{n}}(x)=n-|x|_0$;\vspace{0.3em}\\
(2) if $\frac{n}{100}< |x|_0\leq \frac{n}{50}$,\vspace{-0.2em} \begin{align*}
f^{\mathrm{n}}(x)=\begin{cases}
n-|x|_0 & \text{with probability $1/2+1/\sqrt{n}$},\\
3n+|x|_0 & \text{with probability $1/2-1/\sqrt{n}$};
\end{cases}
\end{align*}
(3) if $\frac{n}{200}< |x|_0\leq \frac{n}{100}$,\vspace{-0.2em} \begin{align*}
f^{\mathrm{n}}(x)=\begin{cases}
4n(n-|x|_0) & \text{with probability $1-1/n$},\\
(2n+|x|_0)^3 & \text{with probability $1/n$};
\end{cases}
\end{align*}
(4) if $|x|_0\leq \frac{n}{200}$,\vspace{-0.8em}
\begin{align*}
f^{\mathrm{n}}(x)=\begin{cases}
n^4(n-|x|_0) & \text{with probability $1/5$},\\
-n^4-\delta & \text{with probability $4/5$},
\end{cases}
\end{align*}
where $\delta$ is randomly drawn from a continuous uniform distribution $\mathcal{U}[0,1]$, and $n/200\in \mathbb{N}^+$.
\end{definition}

\begin{figure*}[t!]\centering
\begin{minipage}[c]{1\linewidth}\centering
        \includegraphics[width=0.9\linewidth, height=0.7\linewidth]{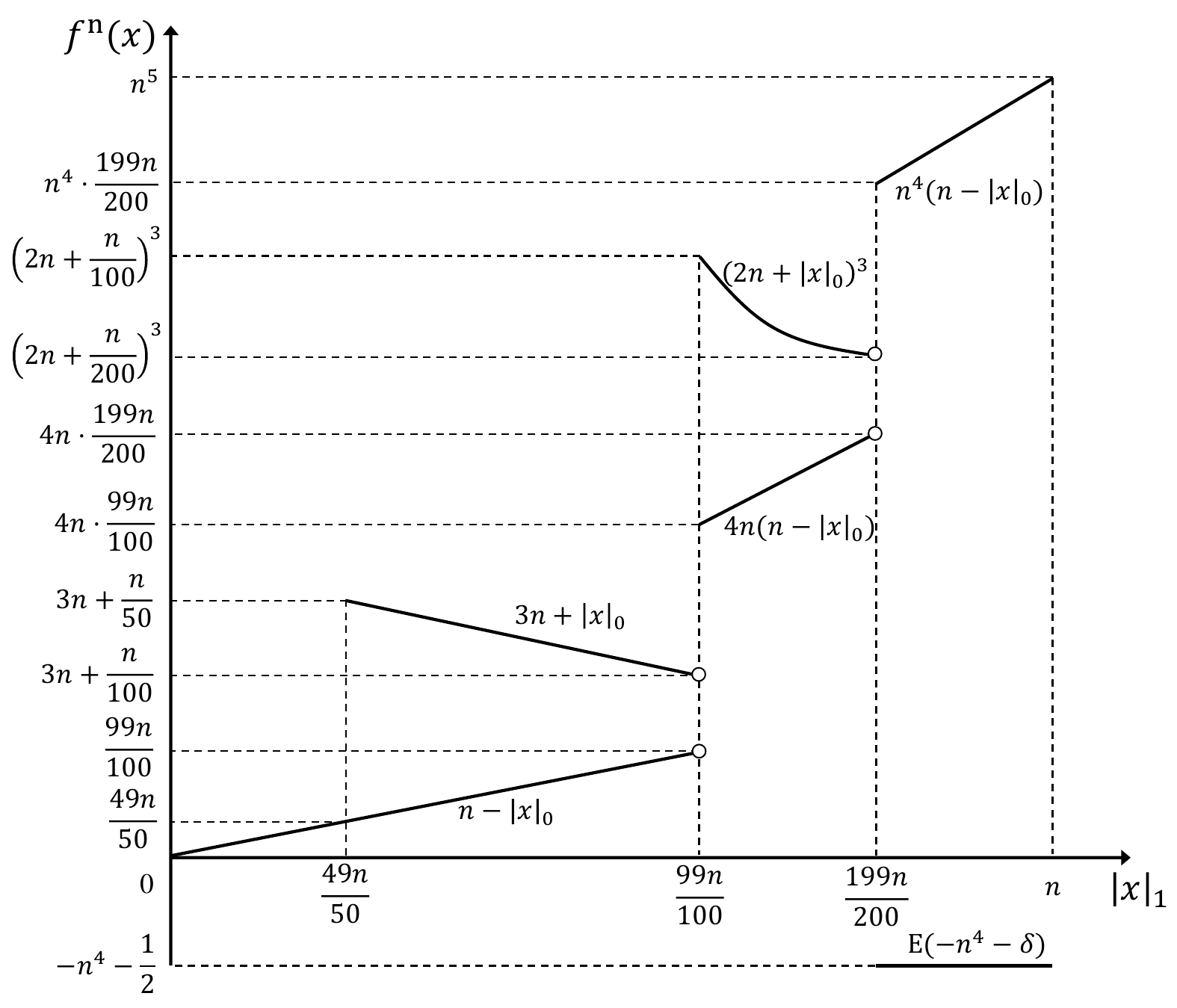}
\end{minipage}
\caption{The seven sub-functions appearing in Definition~\ref{def-seg-noise}. Note that the scale of axes is not strict for plotting all sub-functions clearly.}\label{fig-segmented-noise}
\end{figure*}

We prove in Theorem~\ref{theo-sample-3} that the expected running time of the (1+1)-EA using sampling with any sample size $m$ is exponential. From the proof, we can find the reason for the ineffectiveness of sampling. For two solutions $x$ and $x'$ with $|x'|_0=|x|_0+1$ (i.e., $f(x)=f(x')+1$), the expected gaps between $f^{\mathrm{n}}(x)$ and $f^{\mathrm{n}}(x')$ are positive and negative, respectively, in the segments of $\frac{n}{100}<|x|_0\leq \frac{n}{50}$ and $\frac{n}{200}<|x|_0\leq \frac{n}{100}$. Thus, in the former segment, a larger sample size is better since it will decrease $\mathrm{P}(\hat{f}(x) \leq \hat{f}(x'))$, while in the latter segment, a larger sample size is worse since it will increase $\mathrm{P}(\hat{f}(x) \leq \hat{f}(x'))$. Furthermore, there is no moderate sample size which can make a good tradeoff. Thus, sampling fails in this case. Lemmas~\ref{lemma_berry_esseen} and~\ref{lemma_Bernstein} show the Berry-Esseen and Bernstein inequalities, respectively, which will be used in the proof.

\begin{lemma}[Berry-Esseen Inequality~\cite{tyurin2010improvement}]\label{lemma_berry_esseen}
Let $Z_1,Z_2,\ldots,Z_m$ be i.i.d. random variables with $\mathrm{E}(Z_i)=0$, $\mathrm{Var}(Z_i)=\sigma^2>0$ and $\mathrm{E}(|Z_i|^3)=\rho<+\infty$. It holds that
$$
\mathrm{P}\left(\frac{(\sum^m_{i=1} Z_i/m)\sqrt{m}}{\sigma}\leq x\right)-\Phi(x)\geq -\frac{0.4785\rho}{\sigma^3\sqrt{m}},
$$
where $\Phi(x)$ denotes the cumulative distribution function of the standard normal distribution.
\end{lemma}

\begin{lemma}[Bernstein Inequality~\cite{devroye2012combinatorial}]\label{lemma_Bernstein}
Let $Z_1,Z_2,\ldots,Z_m$ be independent random variables with $\mathrm{E}(Z_i)=0$ and $|Z_i|\le c$ for any $i \in \{1,2,\ldots,m\}$. Let $\sigma^2=\sum_{i=1}^{m}\mathrm{Var}(Z_i)/m$. It holds that for any $t>0$,
$$
\mathrm{P}\left(\sum_{i=1}^{m}Z_i>t\right)\le \mathrm{exp}\left(-\frac{t^2}{2m\sigma^2+2ct/3}\right).
$$
\end{lemma}

\begin{theorem}\label{theo-sample-3}
For the (1+1)-EA solving OneMax under segmented noise, if using sampling with any sample size $m \geq 1$, the expected running time is exponential.
\end{theorem}
\begin{myproof}
We divide the proof into two parts according to the range of $m$. Let $X_t=|x|_0$ denote the number of 0-bits of the solution $x$ maintained by the (1+1)-EA after running $t$ iterations. When $m\leq n^3$, we apply Theorem~\ref{simplified-drift} to prove that starting from $X_0 \geq \frac{n}{50}$, the expected number of iterations until $X_t \leq \frac{n}{100}$ is exponential. When $m> n^3$, we apply Theorem~\ref{simplified-drift} to prove that starting from $X_0 \geq \frac{n}{100}$, the expected number of iterations until $X_t \leq \frac{n}{200}$ is exponential. Due to the uniform initial distribution, both $X_0 \geq \frac{n}{50}$ and $X_0 \geq \frac{n}{100}$ hold with a high probability. Thus, for any $m$, the expected running time until finding the optimum is exponential. For the proof of each part, condition~(2) of Theorem~\ref{simplified-drift} trivially holds as the probability of flipping at least $j$ bits of a solution is at most $2/2^j$, and we only need to show that $\mathrm{E}(X_t-X_{t+1}\mid X_t)$ is upper bounded by a negative constant.
	
\textbf{[Part I: $m\le n^3$]} We consider the interval $[\frac{n}{100},\frac{n}{50}]$. The drift $\mathrm{E}(X_t-X_{t+1} \mid X_t=i)$ (where $\frac{n}{100}<i<\frac{n}{50}$) is calculated as
\begin{align}\label{eq-drift-comp}
&\mathrm{E}(X_t-X_{t+1} \mid X_t=i)=\mathrm{E}^+-\mathrm{E}^-, \quad \text{where} \\
&\mathrm{E}^+=\sum_{x': |x'|_0<i}\mathrm{P}_{\mathrm{mut}}(x,x')\cdot \mathrm{P}(\hat{f}(x') \geq \hat{f}(x)) \cdot (i-|x'|_0),\\
&\mathrm{E}^-=\sum_{x': |x'|_0>i}\mathrm{P}_{\mathrm{mut}}(x,x')\cdot \mathrm{P}(\hat{f}(x') \geq \hat{f}(x))\cdot (|x'|_0-i).
\end{align}
For $\mathrm{E}^-$, we consider the $n-i$ cases where only one 1-bit of $x$ is flipped in mutation. That is, $|x'|_0=i+1$. Next we show that the offspring solution $x'$ is accepted with probability at least $0.07$ (i.e., $\mathrm{P}(\hat{f}(x') \geq \hat{f}(x))\geq 0.07$) by considering two subcases for $m$: (1) $ 4\leq m\le n^3 $ and (2) $1 \leq m\le 3$. In the former case, we mainly apply the Berry-Esseen inequality in Lemma~\ref{lemma_berry_esseen}; in the latter case, the probability $\mathrm{P}(\hat{f}(x') \geq \hat{f}(x))$ can be directly lower bounded.\\
(1) $ 4\leq m\le n^3 $. For $ \frac{n}{100} < k\le \frac{n}{50}$, let $ x^k $ denote a solution with $k$ 0-bits. According to case~(2) of Definition~\ref{def-seg-noise}, we have
\begin{align}\label{eq:mid7}
\mathrm{E}(f^{\mathrm{n}}(x^k))&\!=\!\left(\!\frac{1}{2}\!+\!\frac{1}{\sqrt{n}}\!\right)(n\!-\!k)\!+\!\left(\!\frac{1}{2}\!-\!\frac{1}{\sqrt{n}}\!\right) (3n\!+\!k)=2n\!-\!2\sqrt{n}\!-\!\frac{2k}{\sqrt{n}};\\
\mathrm{Var}(f^{\mathrm{n}}(x^k))&\!=\! \left(\!\frac{1}{2}\!+\!\frac{1}{\sqrt{n}}\!\right)(n\!-\!k)^2 \!+\!\left(\!\frac{1}{2}\!-\!\frac{1}{\sqrt{n}}\!\right) (3n\!+\!k)^2\!-\!\left(2n\!-\!2\sqrt{n}\!-\!\frac{2k}{\sqrt{n}}\right)^2\\
& \!\ge \left(\frac{1}{2}-\frac{1}{\sqrt{n}}\right)\cdot (10n^2+2k^2+4kn)-4n^2\ge n^2,
\end{align}
where the last inequality holds with large enough $n$. Let $Y=f^{\mathrm{n}}(x)-f^{\mathrm{n}}(x')$. Note that $|x|_0 =i \in (\frac{n}{100},\frac{n}{50})$ and $|x'|_0=i+1$. Then, we get that $\mu\coloneqq \mathrm{E}(Y)=\frac{2}{\sqrt{n}}$ and $\sigma^2\coloneqq \mathrm{Var}(Y)\ge 2n^2$. Let $Z=Y-\mu$. Then, we have $\mathrm{E}(Z)=0$, $\mathrm{Var}(Z)=\sigma^2\ge 2n^2$ and
\begin{align}
\rho\coloneqq \mathrm{E}(|Z|^3)&\le 2\left(\frac{1}{4}-\frac{1}{n}\right)\cdot \left(2n+2i+1+\frac{2}{\sqrt{n}}\right)^3\\
&\quad+\left(\left(\frac{1}{2}-\frac{1}{\sqrt{n}}\right)^2+\left(\frac{1}{2}+\frac{1}{\sqrt{n}}\right)^2\right)\cdot \left(1+\frac{2}{\sqrt{n}}\right)^3\le \frac{9n^3}{2},
\end{align}
where the last inequality holds with large enough $n$. Note that $\hat{f}(x)-\hat{f}(x')-\mu$ is the average of $m$ independent random variables, which have the same distribution as $Z$. By Lemma~\ref{lemma_berry_esseen}, we have
\begin{align}\label{mid-eq8}
\mathrm{P}\left(\frac{(\hat{f}(x)-\hat{f}(x')-\mu)\sqrt{m}}{\sigma}\le x\right)-\Phi(x)\geq -\frac{\rho}{2\sigma^3\sqrt{m}},
\end{align}
leading to
\begin{align*}
\mathrm{P}(\hat{f}(x)-\hat{f}(x')\le 0)&=\mathrm{P}(\hat{f}(x)-\hat{f}(x')-\mu\le -\mu)\\
& =\mathrm{P}\left(\frac{(\hat{f}(x)-\hat{f}(x')-\mu)\sqrt{m}}{\sigma}\le -\frac{\mu\sqrt{m}}{\sigma}\right)\\
& \ge \Phi\left(-\frac{\mu\sqrt{m}}{\sigma}\right)-\frac{\rho}{2\sigma^3\sqrt{m}}\\
&\ge \Phi\left(-\frac{\sqrt{2m}}{n\sqrt{n}}\right)-\frac{9}{8\sqrt{2m}},
\end{align*}
where the last inequality is derived by $\mu=\frac{2}{\sqrt{n}}$, $\sigma \geq \sqrt{2}n$ and $\rho \leq \frac{9}{2}n^3$. For $4\le m<n$, $\Phi\left(-\frac{\sqrt{2m}}{n\sqrt{n}}\right)-\frac{9}{8\sqrt{2m}}\geq \Phi(-o(1))-\frac{9}{16\sqrt{2}}\ge 0.07$. For $n\le m\le n^3$, $\Phi\left(-\frac{\sqrt{2m}}{n\sqrt{n}}\right)-\frac{9}{8\sqrt{2m}}\geq \Phi(-\sqrt{2})-o(1)\ge 0.07$. Note that the last inequalities in these two cases both hold with large enough $n$. Thus, we have $\mathrm{P}(\hat{f}(x') \geq \hat{f}(x))\geq 0.07$.\\
(2) $1 \leq m\le 3$. It holds that $\mathrm{P}(\hat{f}(x') \geq \hat{f}(x))\ge (\frac{1}{2}-\frac{1}{\sqrt{n}})^3\ge 0.1 $, since it is sufficient that $f^{\mathrm{n}}(x')$ is always evaluated to $3n+i+1 $ in $m$ independent evaluations. \\
Combining the above two cases, our claim that $\mathrm{P}(\hat{f}(x')\ge \hat{f}(x))\ge 0.07$ holds. Note that $i<n/50$. Thus, we have
$$
\mathrm{E}^-\ge \frac{n-i}{n}\left(1-\frac{1}{n}\right)^{n-1}\cdot 0.07\cdot (i+1-i)\ge \frac{1.2}{50}.
$$
For $\mathrm{E}^+$, we use a trivial upper bound 1 on $\mathrm{P}(\hat{f}(x') \geq \hat{f}(x))$. Then, we have
\begin{align}\label{eq-positive-drift-1}
\mathrm{E}^+& \leq \sum\limits_{x':|x'|_0 <i} \mathrm{P}_{\mathrm{mut}}(x,x')\cdot (i-|x'|_0)\leq \frac{i}{n}\le\frac{1}{50},
\end{align}
where the second inequality can be directly derived from Eq.~(\refeq{eq:mid12}). Thus, the drift satisfies that $$\mathrm{E}(X_t-X_{t+1} \mid X_t=i)=\mathrm{E}^+-\mathrm{E}^- \leq -0.2/50.$$
	
\textbf{[Part II: $m>n^3$]} We consider the interval $[\frac{n}{200},\frac{n}{100}]$, and calculate the drift $\mathrm{E}(X_t-X_{t+1} \mid X_t=i)$ (where $\frac{n}{200}<i<\frac{n}{100}$) by $\mathrm{E}^+-\mathrm{E}^-$ (i.e., Eq.~(\refeq{eq-drift-comp})). For $\mathrm{E}^-$, we show that the probability of accepting the offspring solution $x'$ with $|x'|_0=i+1$ is at least $0.1$. Let $ x^k $ denote a solution with $k$ 0-bits. According to case~(3) of Definition~\ref{def-seg-noise}, we have, for $\frac{n}{200}< k< \frac{n}{100}$,
\begin{align}
&\mathrm{E}(f^{\mathrm{n}}(x^k)-f^{\mathrm{n}}(x^{k+1}))\\
&=\left(1-\frac{1}{n}\right)\cdot 4n-\frac{1}{n}\cdot \left(3(2n+k)^2+3(2n+k)+1\right)\le -8n;
\end{align}
and for $\frac{n}{200}< k\leq \frac{n}{100}$,
\begin{align}
\mathrm{Var}(f^{\mathrm{n}}(x^k)) &=\frac{1}{n}\cdot(2n+k)^6+\left(1-\frac{1}{n}\right)\cdot (4n(n-k))^2-(\mathrm{E}(f^{\mathrm{n}}(x^k)))^2\\
&\le (1/n)\cdot 66n^6+16n^4\le 82n^5.
\end{align}
Then, $\mu:=\mathrm{E}(f^{\mathrm{n}}(x)-f^{\mathrm{n}}(x'))\le -8n$ and  $\sigma^2:=\mathrm{Var}(f^{\mathrm{n}}(x)-f^{\mathrm{n}}(x'))\le 2\cdot 82n^5$. Note that $|f^{\mathrm{n}}(x)-f^{\mathrm{n}}(x')-\mu|\le |f^{\mathrm{n}}(x)-f^{\mathrm{n}}(x')|+|\mu|\le 2(2n+i+1)^3\le 18n^3$. Let $f^{\mathrm{n}}_1(x), f^{\mathrm{n}}_2(x), \ldots, f^{\mathrm{n}}_m(x)$ denote i.i.d. random variables which have the same distribution as $f^{\mathrm{n}}(x)$, and let $f^{\mathrm{n}}_1(x'), f^{\mathrm{n}}_2(x'), \ldots, f^{\mathrm{n}}_m(x')$ denote i.i.d. random variables which have the same distribution as $f^{\mathrm{n}}(x')$. We have
\begin{align}
\mathrm{P}(\hat{f}(x)\ge \hat{f}(x'))
&=\mathrm{P}(m(\hat{f}(x)-\hat{f}(x'))-m\mu\ge -m\mu)\\
&=\mathrm{P}\left(\sum^m_{i=1}(f^{\mathrm{n}}_i(x)-f^{\mathrm{n}}_i(x')-\mu)\ge -m\mu\right)\\
&\le \mathrm{exp}\left(-\frac{m^2\mu^2}{2m\sigma^2+2\cdot 18n^3\cdot (-m\mu)/3}\right)\\
&\le \mathrm{exp}\left(-\frac{m(8n)^2}{2\sigma^2+12n^3\cdot (8n)}\right)< \mathrm{exp}\left(-\frac{n^3\cdot 64n^2}{328n^5+96n^4}\right)\\
&= \mathrm{exp}\left(-\frac{8}{41+o(1)}\right)\le 0.9,
\end{align}
where the second equality holds because $\hat{f}(x)$ and $\hat{f}(x')$ are the average of $m$ independent fitness evaluations of $x$ and $x'$, respectively, the first inequality is by Lemma~\ref{lemma_Bernstein}, the second inequality is by $\mu \leq -8n$, and the third inequality is by $m>n^3$ and $\sigma^2 \leq 2\cdot 82n^5$. Thus, we have $\mathrm{E}^-\ge \frac{n-i}{n}(1-\frac{1}{n})^{n-1}\cdot 0.1\ge \frac{99}{100e}\cdot 0.1\ge 0.03$. For $\mathrm{E}^+$, we still have $\mathrm{E}^+\le \frac{i}{n}\le 0.01$. Thus, the drift satisfies $$\mathrm{E}(X_t-X_{t+1} \mid X_t=i)=\mathrm{E}^+-\mathrm{E}^- \leq -0.02.\vspace{-1.8em}$$\vspace{0.8em}
\end{myproof}

To prove the ineffectiveness of parent populations, we derive a sufficient condition for the exponential running time of the ($\mu$+1)-EA required to solve OneMax under noise, inspired from Theorem~4 in~\cite{friedrich2015benefit}. We generalize their result from additive noise to arbitrary noise. As shown in Lemma~\ref{lemma-population}, the condition intuitively means that when the solution is close to the optimum, the probability of deleting it from the population decreases linearly w.r.t. the population size $\mu$, which is, however, not small enough to make an efficient optimization. Note that for the case where parent populations work in Section~\ref{sec-parent}, the probability of deleting the best solution from the population decreases exponentially w.r.t. $\mu$. Let $poly(n)$ indicate any polynomial of $n$.

\begin{lemma}\label{lemma-population}
For the ($\mu$+1)-EA (where $\mu\in poly(n)$) solving OneMax under noise, if for any solution $y$ with $|y|_1> (599n)/600$ and any set of $\mu$ solutions $Q=\{x^1,x^2,\ldots,x^\mu\}$,
\begin{align}\label{eq:mid5}
\mathrm{P}(f^{\mathrm{n}}(y)<\min\nolimits_{x^i\in Q} f^{\mathrm{n}}(x^i)) \ge 3/(5(\mu+1)),
\end{align}
then the expected running time is exponential.
\end{lemma}
\begin{myproof}
Let $\xi_t$ denote the population after $t$ iterations of the algorithm. Let $X_i^t$ denote the number of solutions with $i$ 1-bits in $\xi_t$. Let $a=\lfloor \frac{599n}{600} \rfloor$ and $b=20$. We first use an inductive proof to show that \begin{align}\label{eq:mid6}\forall t \geq 0, i> a:\; \mathrm{E}(X_i^t)\le \mu b^{a-i}.\end{align} For $t=0$, due to the uniform initial distribution, we have $ \mathrm{E}(X_i^0)=\mu \cdot (\binom{n}{i}/2^n)$. Note that for $j\ge \frac{2n}{3}$, $\binom{n}{j+1}/\binom{n}{j}=\frac{n-j}{j+1}\le \frac{n/3}{2n/3+1}\le \frac{1}{2}$. Thus, for $i> a$, $\binom{n}{i}/2^n\le \binom{n}{\lceil \frac{3n}{4} \rceil +1}/\binom{n}{\lceil \frac{2n}{3} \rceil}\le (\frac{1}{2})^{n/12} \leq b^{a-n}$, which implies that $\forall i > a, \mathrm{E}(X_i^0)\le \mu b^{a-i}$. Next we assume that $ \forall 0\leq t \leq k,i> a: \mathrm{E}(X_i^t)\le \mu b^{a-i}$, and analyze $\mathrm{E}(X_i^{k+1})$ for $i >a$. Let $\bm{X}^k=(X_0^k,X_1^k,...,X_n^k)$, $\bm{l}=(l_0,l_1,...,l_n)$, $|\bm{l}|_1=\sum^{n}_{i=0} l_i$ and $p=\frac{3}{5(\mu+1)}$. Let $x'$ denote the offspring solution generated in the $(t+1)$-th iteration of the algorithm, and let $x^i$ denote any solution with $i$ 1-bits. Let $\mathrm{P}_{\mathrm{mut}}(x,y)$ denote the probability that $x$ is mutated to $y$ by bit-wise mutation. We use $\mathrm{P}_{\mathrm{mut}}(x^j,x^i)=\sum_{y: |y|_1=i}\mathrm{P}_{\mathrm{mut}}(x^j,y)$ to denote the probability of generating a solution with $i$ 1-bits by mutating any solution with $j$ 1-bits. Then, we have
\begin{align}
&\mathrm{E}(X_i^{k+1}-X_i^k)= \mathrm{E}(\mathrm{E}(X_i^{k+1}-X_i^k\mid \bm{X}^k))\\
&=\sum_{|\bm{l}|_1=\mu}\mathrm{P}(\bm{X}^k=\bm{l})\cdot \\
&\qquad \qquad \Big(\mathrm{P}(|x'|_1=i, \text{$x'$ and any $x^i$ in $\xi_k$ are not deleted}\mid \bm{X}^k=\bm{l})\\
& \qquad \qquad \;-\mathrm{P}(|x'|_1\neq i, \text{one $x^i$ in $\xi_k$ is deleted} \mid \bm{X}^k=\bm{l})\Big)\\
&\le \sum_{|\bm{l}|_1=\mu}\mathrm{P}(\bm{X}^k=\bm{l})\cdot \Big(\mathrm{P}(|x'|_1=i\mid \bm{X}^k=\bm{l})\cdot (1-(l_i+1)p)\\
& \qquad \qquad \qquad \qquad \qquad -(1-\mathrm{P}(|x'|_1= i\mid \bm{X}^k=\bm{l}))\cdot l_i p\Big)\\	
&=\sum_{|\bm{l}|_1=\mu}\mathrm{P}(\bm{X}^k=\bm{l})\cdot \big(\mathrm{P}(|x'|_1=i\mid \bm{X}^k=\bm{l})\cdot (1-p)-l_ip\big)\\
&=\sum_{|\bm{l}|_1=\mu}\mathrm{P}(\bm{X}^k=\bm{l}) \left(\sum\limits_{j=0}^{n} \frac{l_j}{\mu}\cdot \mathrm{P}_{\mathrm{mut}}(x^j,x^i)\cdot (1-p)-l_ip \right)\\
&=(1-p)\sum\limits_{j=0}^{n}\mathrm{P}_{\mathrm{mut}}(x^j,x^i)\cdot \sum\limits_{|\bm{l}|_1=\mu}\mathrm{P}(\bm{X}^k=\bm{l})\frac{l_j}{\mu}-\sum\limits_{|\bm{l}|_1=\mu}\mathrm{P}(\bm{X}^k=\bm{l})l_ip \\
&=(1-p)\sum\limits_{j=0}^{n}\mathrm{P}_{\mathrm{mut}}(x^j,x^i)\cdot \sum\limits_{l_j=0}^{\mu}\mathrm{P}(X_j^k=l_j)\frac{l_j}{\mu}-\sum\limits_{l_i=0}^{\mu}\mathrm{P}(X_i^k=l_i)l_i p\\
&=\frac{1-p}{\mu}\cdot \sum\limits_{j=0}^{n}\mathrm{P}_{\mathrm{mut}}(x^j,x^i)\cdot \mathrm{E}(X_j^k)-p\cdot \mathrm{E}(X_i^k),
\end{align}
where the second equality is because $X^{k+1}_i-X^k_i=1$ iff $|x'|=i$ and $x'$ is added into the population meanwhile the solutions with $i$ 1-bits in $\xi_k$ are not deleted; $X^{k+1}_i-X^k_i=-1$ iff $|x'|\neq i$ and one solution with $i$ 1-bits in $\xi_k$ is deleted, the first inequality is because any solution with $i$ 1-bits is deleted with probability at least $p=\frac{3}{5(\mu+1)}$ by Eq.~(\ref{eq:mid5}), and the fourth equality is because a parent solution is uniformly selected from $\xi_k$ for mutation. We further derive an upper bound on $\frac{1}{\mu}\cdot \sum\limits_{j=0}^{n}\mathrm{P}_{\mathrm{mut}}(x^j,x^i)\cdot \mathrm{E}(X_j^k)$ as follows:
\begin{align}
&\frac{1}{\mu}\cdot \sum\limits_{j=0}^{n}\mathrm{P}_{\mathrm{mut}}(x^j,x^i)\cdot \mathrm{E}(X_j^k)\\
&=\frac{1}{\mu}\cdot \left(\sum\limits_{j=0}^{a}+\sum\limits_{j=a+1}^{i-1}+\sum\limits_{j=i}^{i}+\sum\limits_{j=i+1}^{n}\right)\mathrm{P}_{\mathrm{mut}}(x^j,x^i)\cdot \mathrm{E}(X_j^k)\\
&\le \binom{n-a}{i-a}\left(\frac{1}{n}\right)^{i-a}+\sum\limits_{j=a+1}^{i-1}b^{a-j}\cdot \binom{n-j}{i-j}\left(\frac{1}{n}\right)^{i-j}\\
&\quad + b^{a-i}\cdot \left(\left(1-\frac{1}{n}\right)^n + \sum\limits_{l=1}^{n-i}\binom{n-i}{l}\left(\frac{1}{n}\right)^l \right) +  \sum\limits_{j=i+1}^{n}b^{a-j}\\
&\le \left(\frac{n-a}{n}\right)^{i-a}+b^{a-i} \cdot\left(\sum\limits_{j=a+1}^{i-1}b^{i-j}\left(\frac{n-a}{n}\right)^{i-j}\right.\\
&\qquad \qquad \qquad \qquad \qquad \quad \left.+\frac{1}{e}+\sum\limits_{l=1}^{n-i}\left(\frac{n-a}{n}\right)^l + \sum\limits_{j=i+1}^{n}b^{i-j}\right)\\
&\le b^{a-i}\left(\left(\frac{1}{b} \frac{n}{n-a}\right)^{a-i}+\frac{1}{\frac{n}{b(n-a)}-1}+\frac{1}{e}+\frac{1}{\frac{n}{n-a}-1}+\frac{1}{b-1}\right)\\
&\le b^{a-i}/2,
\end{align}
where the first inequality is derived by applying $\forall j\leq a: \mathrm{P}_{\mathrm{mut}}(x^j,x^i) \leq \mathrm{P}_{\mathrm{mut}}(x^a,x^i) \leq \binom{n-a}{i-a}(\frac{1}{n})^{i-a}$, $\sum^n_{j=0}\mathrm{E}(X_j^k)=\mathrm{E}(\sum^n_{j=0}X_j^k)=\mu$, $\forall j > a: \mathrm{E}(X^k_j) \leq \mu b^{a-j}$ and some simple upper bounds on $\mathrm{P}_{\mathrm{mut}}(x^j,x^i)$ for $j > a$, the third inequality is by $\forall 0<c<1: \sum_{l=1}^{+\infty}c^l= \frac{c}{1-c}=\frac{1}{1/c-1}$, and the last holds with $a=\lfloor\frac{599n}{600}\rfloor$, $b=20$, $i>a$ and large enough $n$. Combining the above two formulas, we get $$\mathrm{E}(X_i^{k+1}-X_i^k)\le (1-p)\cdot b^{a-i}/2-p\cdot \mathrm{E}(X_i^k),$$
which implies that
\begin{align}
\mathrm{E}(X_i^{k+1})&\le  (1-p)\cdot b^{a-i}/2+(1-p)\cdot \mathrm{E}(X_i^k)\\
&\le \left(\frac{1}{2\mu}+1\right)\cdot \frac{5\mu+2}{5(\mu+1)}\cdot \mu b^{a-i}\le \mu b^{a-i},
\end{align}
where the second inequality is by $p=\frac{3}{5(\mu+1)}$ and $\mathrm{E}(X_i^k) \leq \mu b^{a-i}$, and the last inequality holds with $\mu\ge 2$. Thus, our claim that $\forall t\geq 0, \forall i>a: \mathrm{E}(X_i^t)\le \mu b^{a-i}$ holds.

Based on Eq.~(\refeq{eq:mid6}) and Markov's inequality, we get, for any $t\geq 0$, $ \mathrm{P}(X_n^t\ge 1)\le \mathrm{E}(X_n^t)\le \mu b^{a-n} $. Note that $X^t_n$ is the number of optimal solutions in the population after $t$ iterations. Let $T=b^{(n-a)/2}$. Then, the probability of finding the optimal solution $1^n$ in $T$ iterations is
\begin{align}
\mathrm{P}(\exists t\le T,X_n^t\ge 1)&\le \sum\limits_{t=0}^{T}\mathrm{P}(X_n^t\ge 1)\le  T\cdot \mu b^{a-n}=\mu\cdot b^{(a-n)/2},
\end{align}
which is exponentially small for $\mu \in poly(n)$. This implies that the expected running time for finding the optimal solution is exponential.\vspace{0.8em}
\end{myproof}

By verifying the condition of Lemma~\ref{lemma-population}, we prove in~Theorem~\ref{theo-parent-adaptive} that the ($\mu$+1)-EA with $\mu\in poly(n)$ needs exponential time for solving OneMax under segmented noise.

\begin{theorem}\label{theo-parent-adaptive}
For the ($\mu$+1)-EA (where $\mu\in poly(n)$) solving OneMax under segmented noise, the expected running time is exponential.
\end{theorem}
\begin{myproof}
We apply Lemma~\ref{lemma-population} to prove this result. For any solution $y$ with $|y|_0\leq n/200 $ and $Q=\{x^1,\ldots,x^\mu\}$, let $ A $ denote the event that $f^{\mathrm{n}}(y)<\min_{x^i \in Q}f^{\mathrm{n}}(x^i)$. We will show that $\mathrm{P}(A)\ge \frac{4}{5(\mu+1)}$, which implies that the condition Eq.~(\refeq{eq:mid5}) holds since $|y|_0\leq n/200$ covers the required range of $|y|_1> 599n/600$.

Let $B_l\; (0\leq l\leq \mu)$ denote the event that $l$ solutions in $Q$ are evaluated to have negative noisy fitness values. Note that for any $x$, $f^{\mathrm{n}}(x)<0$ implies that $ |x|_0\leq n/200$, and $ f^{\mathrm{n}}(x)=-n^4-\delta$ where $\delta \sim \mathcal{U}[0,1]$. For $ 0\le l\le \mu $, $$\mathrm{P}(A\mid B_l)\ge \mathrm{P}(f^{\mathrm{n}}(y)<0\mid B_l)\cdot \mathrm{P}(A \mid f^{\mathrm{n}}(y)<0, B_l).$$
Under the conditions $ f^{\mathrm{n}}(y)<0 $ and $B_l$, the noisy fitness values of $y$ and the corresponding $l$ solutions in $Q$ satisfy the same continuous distribution $-n^4-\delta$ where $\delta \sim \mathcal{U}[0,1]$, thus
$$\mathrm{P}(A\mid f^{\mathrm{n}}(y)<0,B_l)\ge \frac{1}{l+1}\ge \frac{1}{\mu+1}.$$ Then, we get $ \mathrm{P}(A\mid B_l)\ge \frac{4}{5}\cdot \frac{1}{\mu+1}$ and $\mathrm{P}(A)=\sum_{l=0}^{\mu} \mathrm{P}(A\mid B_l)\cdot \mathrm{P}(B_l)\ge \frac{4}{5(\mu+1)}$. By Lemma~\ref{lemma-population}, the theorem holds.\vspace{0.8em}
\end{myproof}

Next we show in Theorem~\ref{theo-offspring-adaptive} that using offspring populations is also ineffective in this case. By using offspring populations, the probability of improving the current fitness becomes very small when the solution is in the 2nd segment (i.e., $\frac{n}{100} < |x|_0 \leq \frac{n}{50}$). This is because a fair number of offspring solutions with fitness no better than the current fitness will be generated with a high probability, and the current fitness becomes better only if all these bad offspring solutions and the parent solution are evaluated correctly, the probability of which almost decreases exponentially w.r.t. $\lambda$. Note that for the (1+$\lambda$)-EA solving OneMax under symmetric noise (i.e., Theorem~\ref{theo-population-offpsring}), the effectiveness of using offspring populations is due to the small probability of losing the current fitness, since it requires a fair number of offspring solutions with fitness no worse than the current fitness to be evaluated incorrectly. Therefore, we can see that using offspring populations can generate a fair number of good and bad offspring solutions simultaneously, and whether it will be effective depends on the concrete noisy problem.

\begin{theorem}\label{theo-offspring-adaptive}
For the (1+$ \lambda $)-EA (where $ \lambda\in poly(n) $) solving OneMax under segmented noise, the expected running time is exponential.
\end{theorem}
\begin{myproof}
We apply the simplified negative drift theorem with scaling (i.e., Theorem~\ref{simplified-drift-scaling}) to prove this result. Let $X_t=|x|_0$ denote the number of 0-bits of the solution $x$ maintained by the (1+$\lambda$)-EA after running $t$ iterations. We consider the interval $[\frac{n}{75},\frac{n}{50}]$, i.e., $a=\frac{n}{75}$ and $b=\frac{n}{50}$ in Theorem~\ref{simplified-drift-scaling}.

First, we analyze $\mathrm{E}(X_{t}-X_{t+1} \mid X_t=i)$ for $\frac{n}{75}< i<\frac{n}{50}$. As the proof of Theorem~\ref{theo-offspring-population-2}, the drift is divided into two parts: $ \mathrm{E^+}= \sum_{j=0}^{i-1}\mathrm{P}(X_{t+1}=j \mid X_t=i)\cdot (i-j)$ and $ \mathrm{E^-}= \sum_{j=i+1}^n\mathrm{P}(X_{t+1}=j \mid X_t=i)\cdot (j-i)$. \\
To analyze $ \mathrm{E^+} $, we will derive upper bounds on $\mathrm{P}(X_{t+1}= j \mid X_t=i)$ separately for two cases: $ \frac{n}{100}< j< i $ and $ 0\le j\le \frac{n}{100} $.\\
(1) $ \frac{n}{100}< j< i $. Let $ q=\sum_{x': |x'|_0\in \{i,i+1\}}\mathrm{P}_{\mathrm{mut}}(x,x') $, i.e., the probability of generating a solution with $i$ or $i+1$ 0-bits by mutating $x$. Since it is sufficient to flip no bits or flip only one 1-bit, $q \geq (1-\frac{1}{n})^n+\frac{n-i}{n}(1-\frac{1}{n})^{n-1}$. Assume that in the reproduction, exactly $ k $ offspring solutions with $ i $ or $ i+1 $ 0-bits are generated, where $ 0\le k\le \lambda $; it happens with probability $\binom{\lambda}{k}\cdot q^k(1-q)^{\lambda-k}$. For $ k=\lambda $, the solution in the next generation must have at least $ i $ 0-bits (i.e., $X_{t+1} \geq i$). For $ 0\le k<\lambda $, each of the remaining $ \lambda-k $ solutions has $ j $ 0-bits with probability $ \frac{p(j)}{1-q}$, where $ p(j):= \sum_{x':|x'|_0=j}\mathrm{P_{mut}}(x,x')$. Thus, under the condition that exactly $k$ offspring solutions with $i$ or $i+1$ 0-bits are generated, the probability that at least one offspring solution has $ j $ 0-bits is $ 1-(1-\frac{p(j)}{1-q})^{\lambda-k} $. Furthermore, to make the solution in the next generation have $ j $ 0-bits (i.e., $X_{t+1}=j$), it is necessary that the fitness evaluation of these $ k $ offspring solutions and the parent solution $ x $ is not affected by noise, the probability of which is $(\frac{1}{2}+\frac{1}{\sqrt{n}})^{k+1}$. Thus, we have, for $\frac{n}{100}<j<i$,
\begin{align}
&\mathrm{P}(X_{t+1}= j \mid X_t=i)\\
&\le  \sum_{k=0}^{\lambda-1}\binom{\lambda}{k}q^k(1-q)^{\lambda-k} \left(1-\left(1-\frac{p(j)}{1-q}\right)^{\lambda-k}\right)\left(\frac{1}{2}+\frac{1}{\sqrt{n}}\right)^{k+1} \\
&\le  \sum_{k=0}^{\lambda-1}\binom{\lambda}{k}\cdot q^k(1-q)^{\lambda-k} \cdot (\lambda-k)\cdot \frac{p(j)}{1-q}\cdot \left(\frac{1}{2}+\frac{1}{\sqrt{n}}\right)^{k+1} \\
&= p(j) \lambda \left(\frac{1}{2}+\frac{1}{\sqrt{n}}\right)\sum_{k=0}^{\lambda-1}\binom{\lambda-1}{k} \left(q \left(\frac{1}{2}+\frac{1}{\sqrt{n}}\right)\right)^k(1-q)^{\lambda-1-k} \\
&= p(j)\lambda\left(\frac{1}{2}+\frac{1}{\sqrt{n}}\right) \left(1-q\cdot \left(\frac{1}{2}-\frac{1}{\sqrt{n}}\right)\right)^{\lambda-1}\le p(j)  \lambda  \left(\frac{2}{3}\right)^{\lambda},
\end{align}
where the last inequality is by $ q \cdot (\frac{1}{2}-\frac{1}{\sqrt{n}})\ge ((1-\frac{1}{n})^{n}+\frac{n-i}{n}(1-\frac{1}{n})^{n-1}) \cdot (\frac{1}{2}-\frac{1}{\sqrt{n}}) \ge \frac{1}{e} \cdot (1-\frac{1}{n}+\frac{49}{50}) (\frac{1}{2}-\frac{1}{\sqrt{n}})\ge \frac{1}{3}$. For $ \lambda\ge 2 $, $ (\lambda+1)\cdot (\frac{2}{3})^{\lambda+1}/(\lambda \cdot (\frac{2}{3})^{\lambda}) = \frac{\lambda+1}{\lambda}\cdot \frac{2}{3}\le 1$, and note that $ 1\cdot \frac{2}{3}\le 1 $ and $ 2\cdot (\frac{2}{3})^2 \le 1$. Thus, for $\frac{n}{100}<j<i$,
\begin{align}\label{eq:mid8}
\mathrm{P}(X_{t+1}=j \mid X_t=i) \le p(j)=\sum_{x':|x'|_0=j}\mathrm{P_{mut}}(x,x').
\end{align}
(2) $ 0\le j\le \frac{n}{100} $. Because to make $X_{t+1}=j$, it is necessary that at least one offspring solution with $j$ 0-bits is generated, we have
\begin{align}\label{eq:mid9}
\mathrm{P}(X_{t+1}=j \mid X_t=i) &\le 1-\left(1-p(j)\right)^\lambda \le \lambda\cdot p(j) \\
&\leq \lambda \cdot \binom{i}{i-j}\frac{1}{n^{i-j}} \le \frac{2\lambda}{2^{i-j}} \le \frac{2\lambda}{2^{n/300}},
\end{align}
where the last inequality is by $ i>\frac{n}{75} $ and $ j\le \frac{n}{100} $.\\
By applying Eqs.~(\refeq{eq:mid8}) and~(\refeq{eq:mid9}) to $\mathrm{E}^+$, we get
\begin{align}
\mathrm{E^+}
& \le \sum_{n/100<j<i} \sum_{x':|x'|_0=j}\mathrm{P_{mut}}(x,x')\cdot (i-j) +\sum_{0\le j\le n/100} \frac{2\lambda}{2^{n/300}}\cdot (i-j)\\
& \le \frac{i}{n}+\frac{2\lambda}{2^{n/300}}\cdot i\cdot \left(\frac{n}{100}+1\right)\le \frac{i+1}{n},
\end{align}
where the second inequality can be directly derived from Eq.~(\refeq{eq:mid12}), and the last holds with $ \lambda\in poly(n) $ and large enough $n$.\\
For $ \mathrm{E^-} $, we have $\mathrm{E}^-=\sum_{j=i+1}^n\mathrm{P}(X_{t+1}=j \mid X_t=i)\cdot (j-i) \geq \mathrm{P}(X_{t+1}\geq i+1 \mid X_t=i)$. To derive a lower bound on $ \mathrm{P}(X_{t+1}\geq i+1 \mid X_t=i)$, it is sufficient that we consider the case where all the $ \lambda $ offspring solutions have more than $ \frac{n}{100} $ 0-bits (denoted as event $ A $). Suppose that $x'$ is generated from $ x $ by mutation, we have $ \mathrm{P}(|x'|_0\le \frac{n}{100})\le \binom{i}{i-\lceil\frac{n}{100}\rceil}\cdot \frac{1}{n^{i-\lceil\frac{n}{100}\rceil}}\le  \frac{1}{(i-\lceil\frac{n}{100}\rceil)!}\le \frac{1}{2^{i-\lceil\frac{n}{100}\rceil-1}}\le \frac{4}{2^\frac{n}{300}} $. Thus, $ \mathrm{P}(A)\ge (1-\frac{4}{2^\frac{n}{300}})^\lambda \ge \frac{3}{4} $, where the last inequality holds with $ \lambda\in poly(n)$ and large enough $n$. Under the condition of $ A $, if one offspring solution has $i+1$ 0-bits (which happens with probability at least $\frac{n-i}{en}$) and its fitness evaluation is affected by noise (which happens with probability $\frac{1}{2}-\frac{1}{\sqrt{n}}$), it must hold that $ X_{t+1} \ge i+1$. Thus, we have
$$ \mathrm{P}(X_{t+1}\ge i+1 \mid X_t=i)\ge \frac{3}{4}\cdot \frac{n-i}{en}\cdot \left(\frac{1}{2}-\frac{1}{\sqrt{n}}\right)\ge \frac{n-i}{8n},$$ implying $$\mathrm{E^-}\ge  (n-i)/(8n).$$
By calculating $\mathrm{E^+}-\mathrm{E^-}$, we get
$$ \mathrm{E}(X_t-X_{t+1} \mid X_t=i)\le (i+1)/n-(n-i)/(8n)\le -1/10,$$
where the last inequality is by $ i<\frac{n}{50}$. Thus, condition~(1) of Theorem~\ref{simplified-drift-scaling} holds with $\epsilon=\frac{1}{10}$.

Next, we examine conditions~(2) and~(3) of Theorem~\ref{simplified-drift-scaling} by setting $ r=\sqrt[3]{n} $. Using the same analysis as Eq.~(\refeq{eq:mid10}) in the proof of Theorem~\ref{theo-offspring-population-2}, we can get, for $j\ge 1 $,
$$ \mathrm{P}(|X_{t+1}-X_{t}|\ge jr \mid X_t \ge 1) \le \frac{2\lambda}{2^{\lfloor jr \rfloor}} \leq \frac{4\lambda}{(2^{\sqrt[3]{n}})^j} \le \frac{1}{e^j},$$
where the last inequality holds with $ \lambda\in poly(n) $ and large enough $n$. Thus, condition~(2) of Theorem~\ref{simplified-drift-scaling} holds. Since $ r=\sqrt[3]{n} $, $ \epsilon =\frac{1}{10} $ and $l=b-a=\frac{n}{150}$, we have $ 1\le r \le \min\{\epsilon^2 l,\sqrt{\epsilon l/(132\ln (\epsilon l))}\}$ for large enough $ n $, and thus condition~(3) of Theorem~\ref{simplified-drift-scaling} also holds.

Note that $ \epsilon l/(132r^2)=\Theta(\sqrt[3]{n}) $ and $ X_0\ge \frac{n}{50}$ holds with a high probability under the uniform initial distribution. Thus, according to Theorem~\ref{simplified-drift-scaling}, we can conclude that the expected running time is exponential.\vspace{0.8em}
\end{myproof}

In the above proof, we apply the simplified negative drift theorem with scaling (i.e., Theorem~\ref{simplified-drift-scaling}) instead of the simplified negative drift theorem (i.e., Theorem~\ref{simplified-drift}). This is because under the condition of a negative constant drift, the requirement on the probability of jumping towards or away from the target state is relaxed by the theorem with scaling, which is easier to be verified in this studied case.

Finally, we prove in Theorem~\ref{theo-adaptive} that the (1+1)-EA using adaptive sampling can solve OneMax under segmented noise in polynomial time. The employed adaptive sampling strategy is defined as follows.
\begin{definition}[Adaptive Sampling]\label{def_adaptive}
To compare two solutions $x, y$, their noisy fitness is first evaluated once independently. If $3n\le |f^{\mathrm{n}}(x)-f^{\mathrm{n}}(y)|< n^4$, this comparison result is directly used (i.e., the sample size $m=1$); otherwise, each solution will be evaluated $5n^3\ln n$ times independently and the comparison will be based on the average value of these $5n^3 \ln n$ fitness evaluations (i.e., the sample size $m=5n^3 \ln n$).
\end{definition}
Intuitively, when the noisy fitness gap of two solutions is too small or too large, we increase the sample size to make a more confident comparison.

To prove Theorem~\ref{theo-adaptive}, we apply the upper bound on the number of iterations of the (1+1)-EA solving noisy OneMax in~\cite{giessen2014robustness}, as presented in Lemma~\ref{lemma-upper}. Let $x^j$ denote any solution with $j$ 0-bits. Lemma~\ref{lemma-upper} intuitively means that if the probability of recognizing the true better solution in the comparison is large, the running time can be upper bounded. From the proof of Theorem~\ref{theo-adaptive}, we can find why adaptive sampling is effective in this case. In the 2nd segment (or the 4th segment) of the noisy problem, $\mathrm{E}(f^{\mathrm{n}}(x)-f^{\mathrm{n}}(y))$ is positive for two solutions $x$ and $y$ with $f(x)>f(y)$, while in the 3rd segment, it is negative. Thus, a large sample size is better in the 2nd and 4th segments, while a small one is better in the 3rd segment. According to the range of the noisy fitness gap of two solutions in each segment, the adaptive sampling strategy happens to allocate $5n^3\ln n$ evaluations for comparing two solutions in the 2nd segment (or the 4th segment), while allocate only one evaluation in the 3rd segment; thus it works.

\begin{lemma}\cite{giessen2014robustness}\label{lemma-upper}
Suppose there is a positive constant $c \le 1/15$ and some $2<l \le n/2$ such that
\begin{align}
&\forall 0<i\le j:\mathrm{P}(\hat{f}(x^j)<\hat{f}(x^{i-1})) \ge 1-l/n;\\
&\forall l<i\le j:\mathrm{P}(\hat{f}(x^j)<\hat{f}(x^{i-1})) \ge 1-ci/n,
\end{align}
then the (1+1)-EA optimizes noisy OneMax in expectation in $O(n\log n)+n2^{O(l)}$ iterations.
\end{lemma}

\begin{theorem}\label{theo-adaptive}
For the (1+1)-EA solving OneMax under segmented noise, if using adaptive sampling in Definition~\ref{def_adaptive}, the expected running time is $O(n^4\log^2n)$.
\end{theorem}
\begin{myproof}
We apply Lemma~\ref{lemma-upper} to prove this result. We will show that $\mathrm{P}(\hat{f}(x^j)\ge \hat{f}(x^{i-1}))$, for all $0<i\le j$, can be upper bounded by $1/n$. As presented in Definition~\ref{def-seg-noise}, $f^{\mathrm{n}}(x)$ can be divided into four segments according to the range of $|x|_0$; in each segment, $f^{\mathrm{n}}(x)$ has a specific expression. Thus, we analyze $\mathrm{P}(\hat{f}(x^j)\ge \hat{f}(x^{i-1}))$ separately by considering $i$ in each segment.\\
(1) $i>\frac{n}{50}$. It holds that $\forall j\ge i$, $\mathrm{P}(\hat{f}(x^{j})\ge \hat{f}(x^{i-1})) =0$, since $f^{\mathrm{n}}(x^j)$ evaluates to the true OneMax fitness and $f^{\mathrm{n}}(x^{i-1})$ must be larger.\\
(2) $\frac{n}{100}+1< i\le\frac{n}{50}$. If $j>\frac{n}{50}$, we easily verify that $ \mathrm{P}(\hat{f}(x^j)\ge \hat{f}(x^{i-1})) =0 $. If $j\le \frac{n}{50}$, $ |f^{\mathrm{n}}(x^j)-f^{\mathrm{n}}(x^{i-1})|<3n $, and thus, both $ x^j $ and $ x^{i-1} $ will be evaluated $m=5n^3\ln n$ times according to the adaptive sampling strategy. Let $Y=f^{\mathrm{n}}(x^{i-1})-f^{\mathrm{n}}(x^j)$. Based on Eq.~(\refeq{eq:mid7}), we easily get $\mu \coloneqq \mathrm{E}(Y)\ge \frac{2}{\sqrt{n}}$. By Hoeffding's inequality, $|f^{\mathrm{n}}(x^{i-1})-f^{\mathrm{n}}(x^j)|<3n$ and $m=5n^3\ln n$, we have $\mathrm{P}(\hat{f}(x^{j})\ge \hat{f}(x^{i-1}))= \mathrm{P}(\hat{f}(x^{i-1})-\hat{f}(x^{j})-\mu \le -\mu)\le \mathrm{exp}\left(-2m\mu^2/(6n)^2\right)\le \mathrm{exp}\left(-40n^3\ln n/(36n^3)\right)\le 1/n$.\\
(3) $\frac{n}{200}+1< i\le\frac{n}{100}+1$. If $j\ge\frac{n}{100}+1$, it holds that $ \mathrm{P}(\hat{f}(x^j)\ge \hat{f}(x^{i-1})) =0 $, since the noisy fitness in the 3rd segment of Definition~\ref{def-seg-noise} is always larger than that in the 2nd segment. If $j\le \frac{n}{100}$, $3n\le  |f^{\mathrm{n}}(x^j)-f^{\mathrm{n}}(x^{i-1})|< n^4 $, and thus, both $ x^j $ and $ x^{i-1} $ are just evaluated once. Then, we get $\mathrm{P}(\hat{f}(x^j)\ge \hat{f}(x^{i-1})) =1/n$, since $\hat{f}(x^j)\ge \hat{f}(x^{i-1})$ iff $\hat{f}(x^j)=(2n+j)^3$. Note that $\hat{f}$ is just $f^{\mathrm{n}}$ here, since it performs only one evaluation.\\
(4) $0<i\le \frac{n}{200}+1$. If $ j>\frac{n}{200} $, $ 0\le f^{\mathrm{n}}(x^j) \le n^4$. Note that $ f^{\mathrm{n}}(x^{i-1})=n^4(n-i+1) $ or $ f^{\mathrm{n}}(x^{i-1}) \le -n^4$. Thus, $ |f^{\mathrm{n}}(x^j) -f^{\mathrm{n}}(x^{i-1})| \ge n^4 $. If $ j\le \frac{n}{200} $, we can easily derive that $ |f^{\mathrm{n}}(x^j) - f^{\mathrm{n}}(x^{i-1})|<n$ or $ \ge n^4 $. Thus, for any $ j\ge i $, both $ x^j $ and $ x^{i-1} $ will be evaluated $m=5n^3\ln n$ times. Let $Y=f^{\mathrm{n}}(x^{i-1})-f^{\mathrm{n}}(x^j)$. It is easy to verify $\mu \coloneqq \mathrm{E}(Y)\ge n^4/5$ and $\sigma^2 \coloneqq \mathrm{Var}(Y)\le 2n^{10}$. By Chebyshev's inequality, $ \mathrm{P}(\hat{f}(x^{j})\ge \hat{f}(x^{i-1}))\le \frac{\sigma^2}{m\mu^2}\le \frac{1}{n}$, where the last inequality holds with large enough $n$.

Thus, it holds that $\forall 0<i\le j: \mathrm{P}(\hat{f}(x^j)\ge \hat{f}(x^{i-1})) \leq 1/n$ for large enough $n$. Let $ l=15 $ and $c=1/15$. The conditions of Lemma~\ref{lemma-upper} are satisfied and the expected number of iterations is thus $ O(n\log n) + O(n)$. Since a solution is evaluated by at most $1+5n^3\ln n$ times in one iteration, the expected running time is $O(n^4\log^2 n)$.
\end{myproof}

\section{Conclusion}\label{sec-conclusion}

In this paper, we analyze the effectiveness of sampling in noisy evolutionary optimization via rigorous running time analysis. First, we construct a family of artificial noisy problems to show that when sampling with any fixed sample size fails, using parent or offspring populations can work. This complements the previous comparison between populations and sampling on the robustness to noise, which only showed the superiority of sampling over populations. Next, through a carefully constructed artificial noisy problem, we show that when using neither sampling nor populations is effective, adaptive sampling which uses a dynamic sample size can work. This provides some theoretical justification for the good empirical performance of adaptive sampling.

From the analysis, we can find that for an optimization problem under noise, if the true fitness order on some solutions is consistent with their expected noisy fitness order while these two orders are reverse on some other solutions, we should be very careful when using the sampling strategy. This is because a consistent order prefers a large sample size while a reverse order requires a small sample size. In such situations, we may use the adaptive sampling strategy, as shown in Section~\ref{sec-adaptive}.

The analysis in Section~\ref{sec-population} shows that parent and offspring populations can bring robustness to noise by making the probability of losing the current best fitness small. For parent populations, losing the current best fitness requires all non-best solutions in the population to appear better. For offspring populations, a fair number of offspring solutions with fitness no worse than the parent solution will be generated, and losing the current fitness requires all these solutions to appear worse. Both events usually occur with a small probability in noisy environments.

We want to point out that this work is only a start for the running time analysis of sampling in noisy evolutionary optimization. All the findings are derived on very artificial noise models. Future work should concentrate on realistic noise models, e.g., additive Gaussian noise. It would be very interesting to examine whether these findings occur in natural noisy situations. Also it would be desirable to analyze the effectiveness of some standard adaptive sampling strategies theoretically.

\begin{acknowledgements}
We want to thank the anonymous reviewers of GECCO'18, TEvC and Algorithmica for their valuable comments and thank Per Kristian Lehre for helpful discussions. This work was supported by the National Key Research and Development Program of China (2017YFB1003102), the NSFC (61672478, 61876077), the Shenzhen Peacock Plan (KQTD2016112514355531), and the Fundamental Research Funds for the Central Universities.
\end{acknowledgements}

\bibliographystyle{spmpsci}      
\bibliography{ectheory}   

\end{document}